\documentclass{article}

\PassOptionsToPackage{numbers, compress}{natbib}

\usepackage[final]{neurips_2022}

\usepackage[utf8]{inputenc} \usepackage[T1]{fontenc}    \usepackage{hyperref}       \usepackage{url}            \usepackage{booktabs}       \usepackage{amsfonts}       \usepackage{nicefrac}       \usepackage{microtype}      \usepackage{xcolor}         \usepackage{xspace}
\usepackage{wrapfig}
\usepackage{csquotes}

\usepackage{enumitem}
\usepackage{algpseudocode}
\usepackage{multirow}
\definecolor{Gray}{gray}{0.9}
\usepackage{hyperref}
\usepackage{algorithm}

\newcommand{\yi}[1]{\textcolor{green}{}}

\newcommand{\ours}{\textsc{Love}\xspace}
\newcommand{\boldtheta}{\boldsymbol{\theta}}

\usepackage{pifont}

\usepackage{amsmath}
\usepackage{amssymb}
\usepackage{mathtools}
\usepackage{amsthm}
\usepackage{amssymb}\usepackage{pifont}\usepackage[font=small]{caption}
\usepackage{subcaption}

\usepackage{hyperref}
\definecolor{myblue}{rgb}{0.0,0.18,0.65}

\hypersetup{ pdftitle={},
pdfauthor={},
pdfsubject={},
pdfkeywords={},
pdfborder=0 0 0,
pdfpagemode=UseNone,
colorlinks=true,
linkcolor=myblue,
citecolor=myblue,
filecolor=myblue,
urlcolor=myblue,    
pdfview=FitH}
\usepackage{url}

\usepackage[capitalize,noabbrev]{cleveref}

\usepackage{bbm}

\def\eqref#1{equation~\ref{#1}}

\def\1{\bm{1}}

\def\rva{{\mathbf{a}}}

\def\rvc{{\mathbf{c}}}

\def\rvh{{\mathbf{h}}}

\def\rvm{{\mathbf{m}}}

\def\rvs{{\mathbf{s}}}

\def\rvx{{\mathbf{x}}}

\def\rvz{{\mathbf{z}}}

\DeclareMathAlphabet{\mathsfit}{\encodingdefault}{\sfdefault}{m}{sl}
\SetMathAlphabet{\mathsfit}{bold}{\encodingdefault}{\sfdefault}{bx}{n}

\def\gA{{\mathcal{A}}}

\def\gD{{\mathcal{D}}}

\def\gH{{\mathcal{H}}}

\def\gL{{\mathcal{L}}}

\def\gP{{\mathcal{P}}}

\def\gR{{\mathcal{R}}}

\def\gX{{\mathcal{X}}}

\def\gZ{{\mathcal{Z}}}

\def\sR{{\mathbb{R}}}

\newcommand{\E}{\mathbb{E}}

\DeclareMathOperator*{\argmax}{arg\,max}
\DeclareMathOperator*{\argmin}{arg\,min}

\newtheorem{theorem}{Theorem}[section]

\newtheorem{proposition}[theorem]{Proposition}

\renewcommand{\mathbf}{\boldsymbol}

\makeatletter
\def\Ddots{\mathinner{\mkern1mu\raise\p@
\vbox{\kern7\p@\hbox{.}}\mkern2mu
\raise4\p@\hbox{.}\mkern2mu\raise7\p@\hbox{.}\mkern1mu}}
\makeatother

\everypar{\looseness=-1}

\title{Learning Options via Compression}

\author{Yiding Jiang\thanks{Equal contribution} \\
Carnegie Mellon University\\
  \texttt{yidingji@cs.cmu.edu} \\
\And
  Evan Zheran Liu\footnotemark[1]\\
  Stanford University \\
  \texttt{evanliu@cs.stanford.edu} \\
\AND
  Benjamin Eysenbach \\
  Carnegie Mellon University \\
  \texttt{beysenba@cs.cmu.edu} \\
  \And
  J. Zico Kolter \\
  Carnegie Mellon University  \\
  \texttt{zkolter@cs.cmu.edu} \\
  \And
  Chelsea Finn \\
  Stanford University \\
  \texttt{cbfinn@cs.stanford.edu} \\
}

\begin{document}
\maketitle
\begin{abstract}
Identifying statistical regularities in solutions to some tasks in multi-task reinforcement learning can accelerate the learning of new tasks.
Skill learning offers one way of identifying these regularities by decomposing pre-collected experiences into a sequence of skills.
A popular approach to skill learning is maximizing the likelihood of the pre-collected experience with latent variable models,
where the latent variables represent the skills.
However, there are often many solutions that maximize the likelihood equally well, including degenerate solutions.
To address this underspecification, we propose a new objective that combines the maximum likelihood objective with a penalty on the description length of the skills. This penalty incentivizes the skills to maximally extract common structures from the experiences.
Empirically, our objective learns skills that solve downstream tasks in fewer samples compared to skills learned from only maximizing likelihood.
Further, while most prior works in the offline multi-task setting focus on tasks with low-dimensional observations, our objective can scale to challenging tasks with high-dimensional image observations.

\end{abstract}

\section{Introduction}
\label{submission}

While learning tasks from scratch with reinforcement learning (RL) is often sample inefficient~\citep{heess2017emergence, berner2019dota}, leveraging datasets of pre-collected experience from various tasks can accelerate the learning of new tasks.
This experience can be used in numerous ways, including 
to learn a dynamics model~\citep{finn2016unsupervised, kidambi2020morel},
to learn compact representations of observations~\citep{yang2021representation},
to learn behavioral priors~\citep{singh2021parrot},
or to extract meaningful skills or options~\citep{sutton1999between, kipf2019compile, ajay2020opal, nam2021skillbased} for hierarchical reinforcement learning (HRL)~\citep{barto2003recent}.
Our work studies this last approach.
The central idea is to extract commonly occurring behaviors from the pre-collected experience as skills, which can then be used in place of primitive low-level actions to accelerate learning new tasks via planning~\citep{roderick2017deep, liu2020learning} or RL~\citep{mcgovern1998macro, lee2018composing, liu2021motor}.
For example, in a navigation domain, learned skills may correspond to navigating to different rooms or objects.

Prior methods learn skills by maximizing the likelihood of the pre-collected experience~\citep{krishnan2017ddco, kipf2019compile, ajay2020opal, zhang2021minimum}.
However, this maximum likelihood objective (or the lower bounds on it) is \emph{underspecified}: it often admits many solutions, only some of which would help learning new tasks (see Figure \ref{fig:diff-decomp}).
For example, one degenerate solution is to learn a single skill for each entire trajectory; another degenerate solution is to learn skills that operate for a single timestep, equivalent to the original action space.
Both of these decompositions can perfectly reconstruct the pre-collected experiences (and, hence, maximize likelihood), but they are of limited use for learning to solve new tasks.
Overall, the maximum likelihood objective cannot distinguish between such decompositions and potentially more useful decompositions, and we find that this underspecification problem can empirically occur even on simple tasks in our experiments.

\begin{figure}
\vspace{-7mm}
\begin{center}
\centerline{\includegraphics[width=\linewidth]{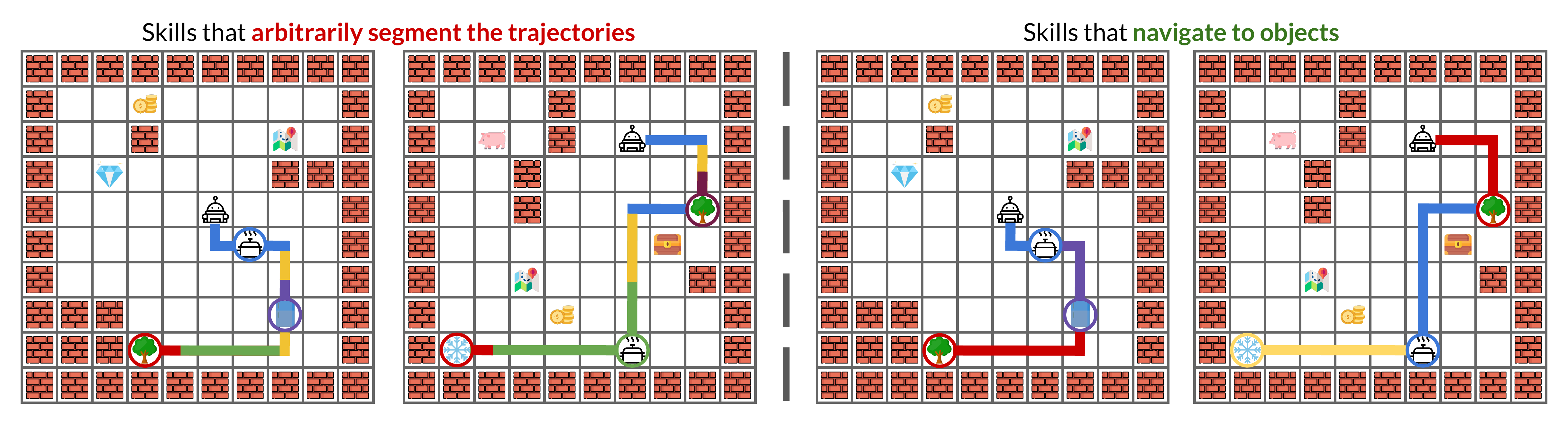}}
\vspace{-1mm}
\caption{
Skill learning via maximizing likelihood is \emph{underspecified}: the skills (represented by the different colored segments) on both the left and the right maximize the likelihood (i.e., encode the data), but the skills on the right are more likely to be useful for a new task.
Our approach factors out common temporal structures, yielding the skills on the right.
The visualized tasks are from \citet{kipf2019compile}.
}
\label{fig:diff-decomp}
\vspace{-7mm}
\end{center}
\end{figure}

To address this underspecification problem, we propose to introduce an additional objective which biases skill learning to acquire skills that can accelerate the learning of new tasks.
To construct such an objective, we make two key observations:
\begin{itemize}[itemsep=0pt,topsep=0pt, leftmargin=10mm]
    \item Skills must extract reusable temporally-extended structure from the pre-collected experiences to accelerate learning of new tasks.
    \item \emph{Compressing} data also requires identifying and extracting common structure.
\end{itemize}
Hence, we hypothesize that compression is a good objective for skill learning.
We formalize this notion of compression with the principle of \emph{minimum description length }(MDL)~\citep{rissanen1978modeling}.
Concretely, we combine the maximum likelihood objective (used in prior works) with a new term to minimize the number of bits required to represent the pre-collected experience with skills, which incentivizes the skills to find common structure.
Additionally, while prior compression-based methods typically involve discrete optimization and hence are not differentiable, we also provide a method to minimize our compression objective via gradient descent (Section \ref{sec:approach}), which enables optimizing our objective with neural networks.

Overall, our main contribution is an objective for extracting skills from offline experience, which addresses the underspecification problem with the maximum likelihood objective.
We call the resulting approach \ours (\textbf{L}earning \textbf{O}ptions \textbf{V}ia compr\textbf{E}ssion).
Using multi-task benchmarks from prior work~\citep{kipf2019compile}, we find that \ours can learn skills that enable faster RL and are more semantically meaningful compared to skills learned with prior methods.
We also find that \ours can scale to tasks with high-dimensional image observations, whereas most prior works focus on tasks with low-dimensional observations.

\section{Related Works}

We study the problem of using offline experience from one set of tasks to quickly learn new tasks.
While we use the experience to learn skills, prior works also consider other approaches of leveraging the offline experience, including to learn a dynamics model~\citep{kidambi2020morel}, to learn a compact representation of observations~\citep{yang2021representation}, and to learn behavioral priors for exploring new tasks~\citep{singh2021parrot}.

We build on a rich literature on extracting skills~\citep{sutton1999between} from \emph{offline} experience~\citep{niekum2013incremental, fox2017multi, krishnan2017ddco, sharma2018directed, kipf2019compile, bera2019podnet, pertsch2020accelerating, zhou2020plas, zhao2020augmenting, lee2020learning, ajay2020opal, Shankar2020Discovering, shankar2020learning, nam2021skillbased, zhang2021minimum, lu2021learning, zhu2021bottom, rao2021learning, tanneberg2021skid, yang2021trail}.
These works predominantly learn skills using a latent variable model, where the latent variables partition the experience into skills, and the overall model is learned by maximizing (a lower bound on) the likelihood of the experiences.
This approach is structurally similar to a long line of work in hierarchical Bayesian modeling~\citep{ghahramani2000variational, blei2001topic, fox2011sticky, linderman2016recurrent, johnson2016composing, dai2016recurrent, linderman2017bayesian, kim2019variational, harrison2019continuous, dong2020collapsed}.
We also follow this approach and build off of VTA~\citep{kim2019variational} for the latent variable model,
but differ by introducing a new compression term to address the underspecification problem with maximizing likelihood.
Several prior approaches~\citep{zhang2021minimum, zhao2020augmenting, garcia2017compression} also use compression, but yield open-loop skills, whereas we aim to learn closed-loop skills that can react to the state.
Additionally, \citet{zhang2021minimum} show that maximizing likelihood with variational inference can itself be seen as a form of compression, but this work compresses the latent variables at each time step, which does not in general yield optimal compression as our objective does.
Other works learn skills by limiting the information used from the state~\citep{goyal2019reinforcement}, whereas we learn skills by compressing entire trajectories.

Beyond learning skills from offline experience, prior works also consider learning skills with additional supervision~\citep{andreas2017modular, oh2017zero, shiarlis2018taco, jiang2019language, sontakke2021video2skill}, from online interaction without reward labels~\citep{gregor2016variational, florensa2017stochastic, eysenbach2018diversity, Sharma2020Dynamics, bagaria2021skill}, and from online interaction with reward labels~\cite{thrun1994finding, kulkarni2016hierarchical, bacon2017option, vezhnevets2017feudal, hausman2018learning, nachum2018data}.
Additionally, a large body of work on meta-reinforcement learning~\citep{duan2016rl, wang2016learning, rakelly2019efficient, zintgraf2019varibad, liu2021decoupling} also leverages prior experience to quickly learn new tasks, though not necessarily through skill learning.

\section{Preliminaries}
\label{sec:prelim}
\paragraph{Problem setting.} We consider the problem of using an offline dataset of experience to quickly solve new RL tasks, where each task is identified by a reward function.
The offline dataset is a set of trajectories $\mathcal{D} = \{\tau_i\}_{i = 1}^N$, where each trajectory is a sequence of states $\rvx \in \gX$ and actions $\rva \in \gA$: $\tau_i = \{(\rvx_1, \rva_1), (\rvx_2, \rva_2), \ldots\}$.
Each trajectory is collected by some unknown policy interacting with a Markov decision process (MDP) with dynamics $\gP(\rvx_{t + 1} \mid \rvx_{t}, \rva_t)$.
Following prior work~\citep{finn2016unsupervised, agrawal2016learning, dasari2019robonet}, we do not assume access to the rewards (i.e., the task) of the offline trajectories.

Using this dataset, our aim is to quickly solve a new task.
The new task involves interacting with the same MDP as the data collection policy in the offline dataset, except with a new reward function $\gR(\rvx, \rva)$.
The objective is to learn a policy that maximizes the expected returns in as few numbers of environment interactions as possible.

\paragraph{Variational Temporal Abstraction.} Our method builds upon the graphical model from variational temporal abstraction (VTA)~\citep{kim2019variational},
which is a method for decomposing non-control sequential data (i.e., data without actions) into subsequences.
We overview VTA below and extend it to handle actions it Section \ref{sec:vta_actions}.

\begin{wrapfigure}[15]{r}{0.387\textwidth}
\begin{center}
\vspace{-7mm}
\centerline{\includegraphics[width=0.38\textwidth]{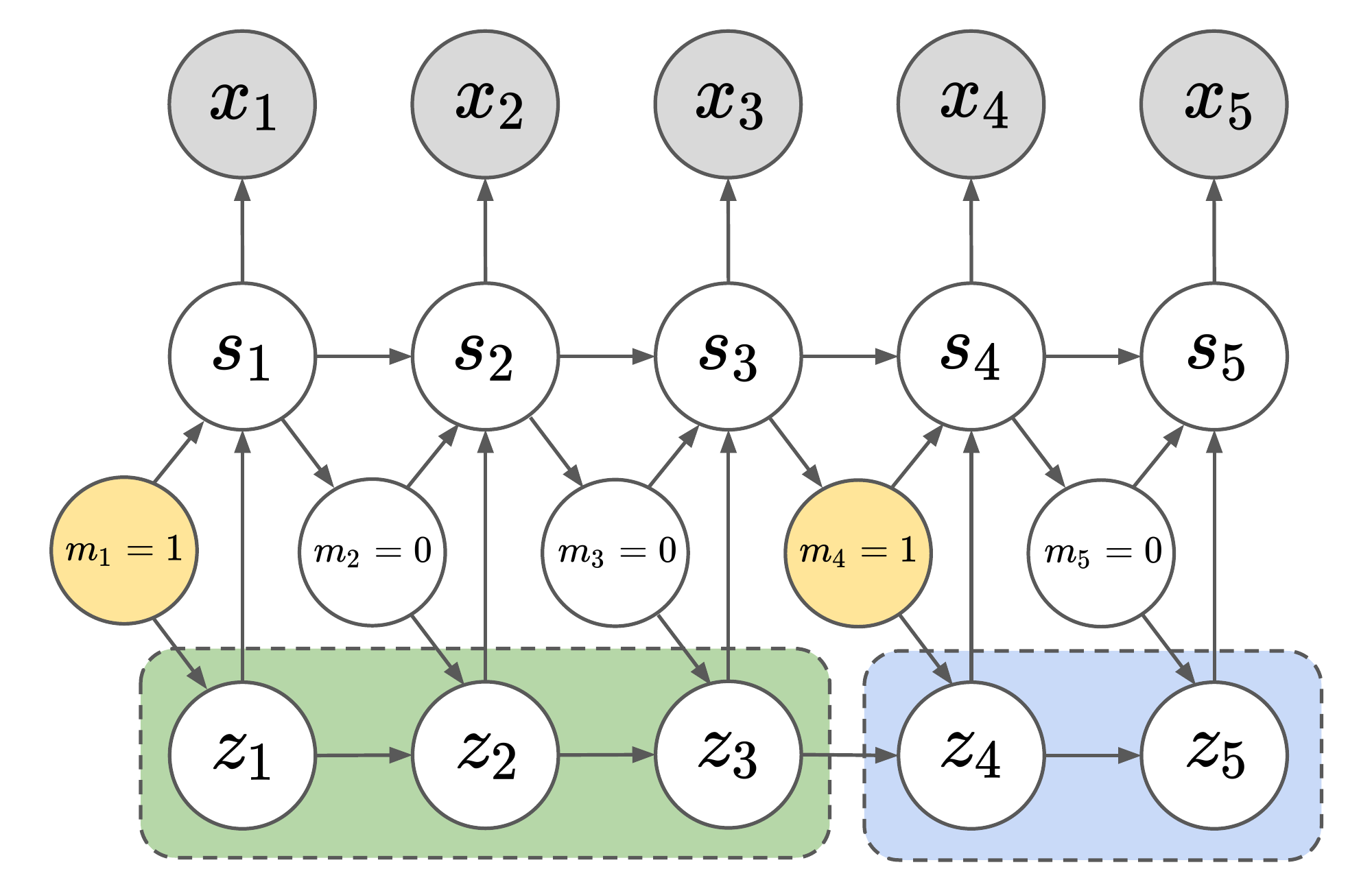}}
\vspace{-1mm}
\caption{\footnotesize The VTA graphical model (adapted from \citet{kim2019variational}).
The model decomposes the data $x_{1:T}$ into subsequences demarcated by when the boundary variables $m_t = 1$ (highlighted in \textcolor[HTML]{bf9000}{yellow}).
Each subsequence is assigned a descriptor $\rvz$, where all $\rvz_t$ within a subsequence are the same (outlined in dashed lines).
}
\label{fig:pgm}
\end{center}
\end{wrapfigure}

VTA assumes that each observation $\rvx_t$ is associated with an abstract representation $\rvs_t$ and a subsequence descriptor $\rvz_t$. An additional, binary random variable $\rvm_t$ indicates whether a new subsequence begins at the current observation.
VTA uses the following generative model:
\begin{enumerate}[itemsep=0mm, topsep=0mm, leftmargin=10mm]
    \item Determine whether to begin a new subsequence: $\rvm_t~\sim~p(\rvm_t \mid \rvs_{t - 1})$. A new subsequence always begins on the first time step, i.e., $\rvm_1 = 1$.
    \item Sample the descriptor:    $\rvz_{t} \sim p(\rvz_{t} \mid \rvz_{<t}, \rvm_t)$. If $\rvm_t = 0$, the previous descriptor is copied, i.e., $\rvz_t~=~\rvz_{t - 1}$.
    \item Sample the next state abstract representation:\\ $\rvs_{t} \sim p(\rvs_t \mid \rvs_{<t}, \rvz_t, \rvm_t) $.
    \item Sample the observation: $\rvx_{t} \sim p(\rvx_{t} \mid \rvs_{t})$.
\end{enumerate}
We visualize this sampling procedure in Figure~\ref{fig:pgm}. Formally, we can write this generative model as:
\begin{align*}
     p(\rvx_{1:T}, \rvz_{1:T}, \rvs_{1:T}, \rvm_{1:T}) 
    = \prod_{t=1}^T \,\, p(\rvx_t \mid \rvs_t) p(\rvm_t \mid \rvs_{t - 1}) p(\rvs_t \mid \rvs_{<t}, \rvz_t, \rvm_t) 
     p(\rvz_t \mid \rvz_{<t}, \rvm_t).
\end{align*}

VTA maximizes the likelihood of the observed data under this latent variable model using the standard evidence lower bound (ELBO):
\begin{align*} \footnotesize
    \log p(\rvx_{1:T}) \geq \sum_{\rvm_{1:T}} \int_{\substack{\rvs_{1:T}, \\ \rvz_{1:T}}} q_{\boldsymbol{\theta}}(\rvz_{1:T}, \rvs_{1:T}, \rvm_{1:T} \mid \rvx_{1:T}) \cdot \log \frac{p_{\boldsymbol{\theta}}(\rvx_{1:T}, \rvz_{1:T}, \rvs_{1:T}, \rvm_{1:T})}{q_{\boldsymbol{\theta}} (\rvz_{1:T}, \rvs_{1:T}, \rvm_{1:T} \mid \rvx_{1:T})}
\end{align*}
This lower bound holds for any choice of $q_{\boldtheta}$. VTA chooses to factor this distribution as:
\begin{small}
 \begin{align*}
    & q_{\boldsymbol{\theta}}(\rvz_{1:T}, \rvs_{1:T}, \rvm_{1:T} \mid \rvx_{1:T}) = \prod_{t=1}^T q_{\boldsymbol{\theta}}(\rvm_t \mid \rvx_{1:t}) \, q_{\boldsymbol{\theta}}(\rvz_t \mid \rvz_{t-1}, \rvm_t, \rvx_{1:T})\,
     q_{\boldsymbol{\theta}}(\rvs_t \mid \rvz_t, \rvm_t, \rvx_{1:t}).
\end{align*}
\end{small}

As discussed above and as we observe in Section \ref{sec:didactic}, the maximum likelihood objective is underspecified and may yield degenerate or unhelpful solutions, which we address in the next section.
Specifically, we use the graphical model from VTA, but introduce a new objective that compresses the latent variables $\rvz$, while also maximizing the ELBO.

\section{\ours: Learning Options via Compression}\label{sec:approach}

In this section, we describe our method for learning skills from pre-collected experience.
We first extend the VTA graphical model to handle experience labeled with actions, and then introduce a compression objective that encourages extracting common structure from the experience.

\subsection{A Graphical Model for Interaction Data}\label{sec:vta_actions}
We now extend the VTA graphical model~\citep{kim2019variational} to handle sequential data labeled with actions, where descriptors $\rvz$ now represent skills.
From a high level, the model partitions a trajectory into subsequences with the boundary variables $\rvm$ and labels each subsequence as a skill $\rvz$.

We wish to learn a state-conditional policy rather than the joint distribution of the whole trajectory.
To do this, we write a new generative model for the actions $\rva_{1:T}$ conditional on the state $\rvx_{1:T}$:
\begin{align*}
      p(\rvz_{1:T}, \rvs_{1:T}, \rvm_{1:T}, \rva_{1:T} \mid \rvx_{1:T}) 
     = \prod_{t=1}^T \,\, p(\rva_t \mid \rvs_t) p(\rvm_t \mid \rvs_{t - 1}) p(\rvs_t \mid \rvx_{t}, \rvz_t) 
     p(\rvz_t \mid \rvx_{t}, \rvz_{t - 1}, \rvm_{t-1}).
\end{align*}
This differs from the original VTA generative model in two ways:
(1) we introduce a $p(\rva_t \mid \rvs_t)$ term that indicates how actions are sampled;
and (2) the distributions over $\rvz_t$ and $\rvs_t$ do not depend on all previous $\rvz_{1:t}$ and $\rvs_{1:t}$ to encode the Markov property.
We then augment the variational distribution such that the posterior over skills $\rvz$ also depends on actions:
\begin{small}
\begin{align*}
     q_{\mathbf{\theta}}(\rvz_{1:T}, \rvs_{1:T}, \rvm_{1:T} \mid \rvx_{1:T}, \rva_{1:T}) =  \prod_{t=1}^T \underbrace{ q_{\mathbf{\theta}}(\rvm_t \mid \rvx_{1:t})}_{\substack{\text{termination} \\ \text{policy}}} \,\underbrace{ q_{\mathbf{\theta}}(\rvz_t \mid \rvz_{t - 1}, \rvm_t, \rvx_{1:T}, \rva_{1:T})}_{\text{skill posterior}}
    \underbrace{ q_{\mathbf{\theta}}(\rvs_t \mid \rvz_t, \rvx_t)}_{\text{state abstraction posterior}}.
\end{align*}
\end{small}

Overall, this yields a model with 3 learned components:
\begin{itemize}[itemsep=0pt, topsep=0pt, leftmargin=10mm]
    \item A \emph{state abstraction posterior} $q_{\mathbf{\theta}}(\rvs_t \mid \rvz_t, \rvx_t)$ and \emph{decoder} $p_{\mathbf{\theta}}(\rva_t \mid \rvs_t)$, which together form a distribution over actions conditioned on the skill $\rvz_t$ and current state $\rvx_t$.
    We refer to these jointly as the \emph{skill policy} $\pi_{\mathbf{\theta}}(\rva_t \mid \rvz_t, \rvx_t) = \E_{\rvs_t \sim q_{\mathbf{\theta}}(\rvs_t \mid \rvz_t, \rvx_t)}\left[p_{\mathbf{\theta}}(\rva_t \mid \rvs_t)\right]$.
    \item A \emph{termination policy} $q_{\mathbf{\theta}}(\rvm_t \mid \rvx_{1:t})$, which determines whether the previous skill terminates.
    \item A \emph{skill posterior} $q_{\mathbf{\theta}}(\rvz_t \mid \rvz_{t - 1}, \rvm_t, \rvx_{1:T}, \rva_{1:T})$, which outputs the current skill $\rvz_t$ conditioned on all states $\rvx_{1:T}$ and actions $\rva_{1:T}$.
    This depends on $\rvm_t$ and $\rvz_{t - 1}$, since the boundary variables determine whether the previous skill $\rvz_{t - 1}$ terminates: when $\rvm_t = 0$, the previous skill does not terminate and $\rvz_t = \rvz_{t - 1}$.
\end{itemize}

Once this model is learned, the skill policy and termination policy represent the skills, without a need for the skill posterior:
Given a skill variable $\rvz$, the skill policy encodes what actions the skill takes, and the termination policy determines when the skill stops.
In the next section, we describe our objective for learning this model.
Then, in Section \ref{sec:hrl}, we describe how we use the skill policy and termination policy to learn new tasks.

\subsection{Discovering Structure via Compression}\label{sec:compression_objective}

As previously discussed, the maximum likelihood objective is underspecified for skill learning, because many skills can maximize likelihood, independent of whether they are useful for learning new tasks.
In this section, we address this with a new objective that attempts to measure how useful skills will be for learning new tasks in terms of compression.
Our objective is based on the intuition that effectively compressing a sequence of data requires factoring out common structure, and factoring out common structure is critical for learning useful skills.

We measure the complexity of a skill decomposition as the amount of information required to communicate the sequence of skills that encode the pre-collected experience.
The latent variable model introduced in the previous section encodes each trajectory as a sequence of skills $\rvz_{1:T}$ and boundaries $\rvm_{1:T}$.
However, using the skill and termination policies, it is possible to recover each trajectory from only the skill variables at the boundary points: i.e., at the time steps $\{t \in [T] \, \mid \, \rvm_t = 1\}$.
For a prior on skills $p_{\rvz}$, an optimal code requires $-\log p_{\rvz}(\rvz_t)$ bits to send a skill $\rvz_t$~\citep[Chapter 5.2]{cover1999elements}.
Hence, the expected code length of communicating a trajectory $\mathbf{\tau}_{1:T}~=~\{(\rvx_1, \rva_1), \ldots, (\rvx_T, \rva_T)\}$ is:
\begin{equation*}
    \textsc{InfoCost}(\mathbf{\theta}; p_{\rvz}) = -\E_{\substack{\mathbf{\tau}_{1:T},\\\rvm_{1:T},\\\rvz_{1:T}}}\left[\sum_{t=1}^T \log p_{\rvz}(\rvz_t) \rvm_t\right],
\end{equation*}
where the expectation is under trajectories $\mathbf{\tau}_{1:T}$ from the pre-collected experience and sampling $\rvz_{1:T}$ from the skill posterior and $\rvm_{1:T}$ from the termination policy.
The choice of prior that minimizes the average code length is one that equals the empirical distribution of skills under the pre-collected experience~\citep[Chapter 5.3]{cover1999elements}:
\begin{align*}
    p_{\rvz}^{\star} \triangleq \argmin_{p_{\rvz}} \textsc{InfoCost}( \mathbf{\theta}; p_{\rvz}), 
\quad \text{where } p_\rvz^{\star}(\rvz) = \frac{1}{n_\text{s}} \E_{\substack{\mathbf{\tau}_{1:T},\\\rvm_{1:T},\\\rvz_{1:T}}}\left[\sum_{t=1}^T \delta(\rvz_t = \rvz)\rvm_t\right],
\end{align*}
and $n_\text{s} \triangleq \E_{\substack{\mathbf{\tau}_{1:T},\rvm_{1:T}}}\left[\sum_{t=1}^{T}\rvm_t\right]$.
In Appendix \ref{app:valid_density}, we show that this marginal $p_\rvz^{\star}$ is a proper density for both continuous and discrete $\rvz$.
Substituting in this optimal choice of prior, we can show that the code length can be expressed as the marginal entropy $\gH_{p_{\rvz}^{\star}}[\rvz]$ times the number of skills per trajectory $n_\text{s}$ (see Appendix \ref{app:rewrite} for proof):
{
\begin{align}\label{eqn:compression_objective}
\footnotesize
\gL_{\text{CL}}(\mathbf{\theta}) \triangleq \min_{p(\rvz_t)} \textsc{InfoCost}(\mathbf{\theta}; p(\rvz_t)) 
 = -\E_{\substack{\mathbf{\tau}_{1:T},\\\rvm_{1:T},\\\rvz_{1:T}}}\left[\sum_{t=1}^T \log p_{\rvz}^{\star}(\rvz_t) \rvm_t\right] = n_\text{s}\gH_{p_{\rvz}^{\star}}[\rvz].
\end{align}
}Intuitively, this is equal to the average code length of a skill multiplied by the the average number of skills per trajectory.
Note that compression is only beneficial if the model also achieves high likelihood of the data.
We capture this by solving the following constrained optimization problem:
\begin{equation}\label{eqn:constrained_optimization}
    \min_{\boldsymbol{\theta}}\,\, \gL_{\text{CL}}( \boldsymbol{\theta}) \quad
    \text{s.t.} \,\, \gL_{\text{ELBO}}(\boldsymbol{\theta}) \leq C,
\end{equation}
where $\gL_{\text{ELBO}}(\boldsymbol{\theta})$ a negated evidence lower bound on the log-likelihood (detailed in Appendix \ref{app:elbo}).
\paragraph{Remarks.}
We discuss a connection between this objective and variational inference in Appendix~\ref{app:prior}.
Additionally, while this work focuses on the RL setting, our objective generally applies to sequential modeling problems. We believe that it could be useful for many applications beyond option learning.

\subsection{Connections to Minimal Description Length}
Our approach closely relates to the minimum description length (MDL) principle~\citep{rissanen1978modeling}.
This principle equates \emph{learning} with \emph{finding regularity} in data, which can be used to \emph{compress} the data.
Informally, the best model is the one that encodes the data with the lowest description length.
Given a model $\mathbf{\theta}$ that encodes the data $\mathcal{D}$ into some message, one way to formalize the description length of the data $L(\mathcal{D})$ is with a crude two-part code~\citep{grunwald2004tutorial}.
This decomposes the description length of the data as the length of the message plus the length of the model:
\begin{align}\label{eqn:mdl}
    L(\gD) = \underbrace{L(\gD \mid \boldsymbol{\theta})}_{\text{message length}} + \underbrace{L(\boldsymbol{\theta})}_{\text{model length}}.
\end{align}
In our case, the message is the latent variables $\rvz_t$ at the boundary points $\{t \in [T] \, \mid \, \rvm_t = 1\}$, which can be decoded into the data $\mathcal{D}$ with the skill and termination policies (representing the model).
Hence, our approach can be seen as an instance of minimizing the description length $L(\mathcal{D})$.
Optimizing our objective in Equation~\ref{eqn:compression_objective} is equivalent to minimizing the message length $L(\gD \mid \boldsymbol{\theta})$.
While we do not directly attempt to minimize the model length term $L(\mathbf{\theta})$,
many works~\citep{neyshabur2014search, valle2018deep} indicate that deep learning has \emph{implicit regularization}, which biases optimization toward low complexity solutions without an explicit regularizer.
In general, computing the true description length or appropriate notion of complexity for neural networks is a tall order~\citep{neyshabur2017exploring, zhang2021understanding, Jiang*2020Fantastic};
however, there is a rich space of methods to be explored here and we therefore leave the extension of our approach to directly minimizing the model length term for future work.

\subsection{A Practical Implementation}\label{sec:practical}
\paragraph{Model.}
We instantiate our model by defining discrete skills $\rvz \in [K]$, state abstractions $\rvs \in \sR^d$, and binary boundary variables $\rvm \in \{0, 1\}$.
We parameterize all components of our model as neural networks.
See Appendix \ref{app:arch} for architecture details.

\paragraph{Optimization.} 
We apply Gumbel-softmax~\citep{jang2016categorical, maddison2016concrete} to optimize over the discrete random variables $\rvz$ and $\rvm$.
We rewrite the compression objective $\gL_\text{CL}$ as a product of $n_\text{s}$ and $\gH_{p_{\rvz}^{\star}}[\rvz_t]$ in Equation \ref{eqn:compression_objective} to improve the stability of optimization.
This rewriting allows computing a gradient in terms of a distribution over the skill variables $p^\star_\rvz$, rather than samples $\rvz_{t}$, which yields a more accurate finite sample gradient.
Empirically, this leads to stabler optimization and convergence to better solutions.

We approximate the optimal skill prior $p_{\rvz}^{\star}$ using our skill posterior as:
\begin{align*}\footnotesize
    p_{\rvz}^{\star}(\rvz) \approx\frac{1}{ \widehat{\E}_{\substack{\mathbf{\tau}_{1:T}, \\\rvm_{1:T}}} \left[\sum_{t = 1}^T \rvm_t\right]} \widehat{\E}_{\substack{\mathbf{\tau}_{1:T},\\\rvm_{1:T},\\\rvz_{1:T}}} \left[\sum_{t = 1}^T q_{\boldtheta}(\rvz_t = \rvz \mid \rvz_{t-1}, \rvm_t, \rvx_{1:T}, \rva_{1:T})\,\rvm_t\right] ,
\end{align*}
where $\widehat{\E}$ denotes the empirical expectation over minibatches of $\rvx_{1:T}, \rva_{1:T}$ from $\gD$ and sampling $\rvm_{1:T}$ from the termination policy.
We solve the constrained optimization problem (Equation \ref{eqn:constrained_optimization}) with dual gradient descent on the Lagrangian.
In Appendix \ref{app:algo}, we summarize the overall training procedure in Algorithm \ref{alg:love} and report details about the Lagrangian.

\paragraph{Enforcing a minimum skill length.} Though degenerate solutions such as skills that only take a single action score poorly in our compression objective, empirically, such solutions create local optima that are difficult to escape.
To avoid this, we mask the boundary variables during training to ensure that each skill $z_t$ operates for at least $T_\text{min} = 3$ time steps.
We find that when the skills are at least minimally temporally extended, optimization of the compression objective appears to be stabler and achieves better values.
We remove these masks at test time, when we learn a task with our skills.

\section{Using the Learned Skills for Hierarchical RL}
\label{sec:hrl}
We now describe how we quickly learn new tasks, given the skills learned from the pre-collected experience.
Overall, we simply augment the agent's action space with the learned skills and learn a new policy that can take either low-level actions or learned skills, similar to \citet{kipf2019compile}.
However, our compression objective may result in unused skills, since this decreases the encoding cost of the used skills.
Therefore, we first filter down to only the skills that are used to compress the pre-collected experience by selecting the skills where the marginal is over some threshold $\alpha$:
\begin{align*}
    \gZ = \big\{k \in [K] \,\,\mid\,\,  p_{\rvz}^{\star}(k) > \alpha \big\}.
\end{align*}
Then, on a new task with action space $\mathcal{A}$, we train an agent using the augmented action space $\gA^+ = \gA \, \cup \, \gZ$.
When the agent selects a skill $z \in \gZ$, we follow the procedure in Algorithm~\ref{alg:hrl} in the Appendix.
We take actions following the skill policy $\pi_{\mathbf{\theta}}(\rva_t \mid \rvz, \rvx_t)$ (lines 2--3), until the termination policy $q_{\mathbf{\theta}}(\rvm_t \mid \rvx_{1:t})$ outputs a termination $\rvm_t = 1$ (line 5).
At that point, the agent observes the next state $\rvx_{t + 1}$ and selects the next action or skill.

\section{A Didactic Experiment: Frame Prediction}\label{sec:didactic}

Before studying the RL setting, we first illustrate the effect of \ours in the simpler setting of sequential data without actions.
In this setting, we use the original VTA model, which partitions a sequence $\rvx_{1:T}$ into contiguous subsequences with the boundary variables $\rvm$ and labels each subsequence with a descriptor $\rvz$, as described in Section \ref{sec:prelim}.
Here, VTA's objective is to maximize the likelihood of the sequence $\rvx_{1:T}$, and our compression objective applies to this setting completely unmodified.
We compare \ours and only maximizing likelihood, i.e., VTA~\citep{kim2019variational}, by measuring how well the learned subsequences correspond to underlying patterns in the data.
To isolate the effect of the compression objective, we do not enforce a minimum skill length.

\begin{figure}
    \vspace{-7mm}
    \begin{minipage}{0.43\linewidth}
        \centering
        \includegraphics[width=\linewidth]{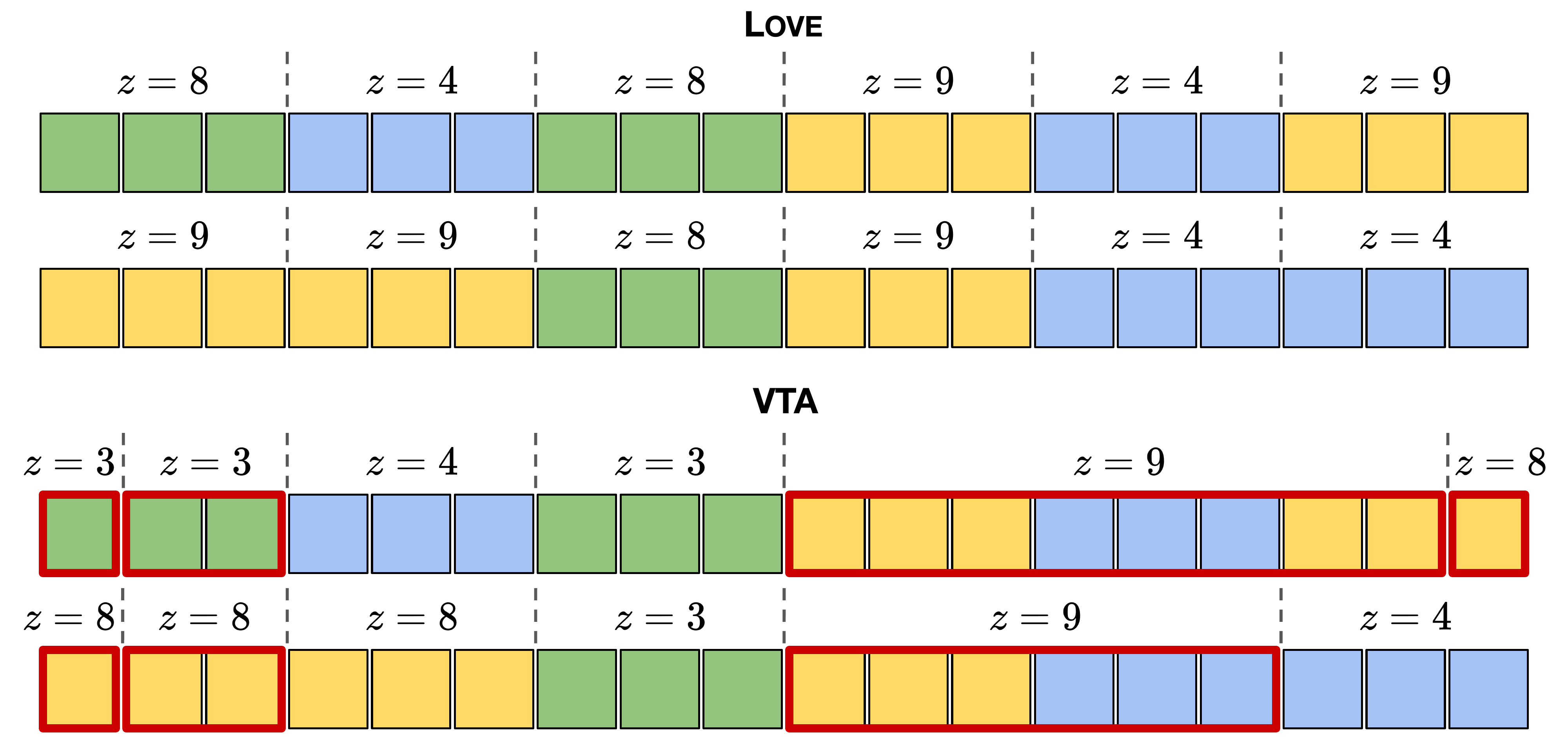}
\caption{\footnotesize Learned boundaries and descriptors on \emph{Simple Colors}.
        \ours recovers the patterns, while VTA learns boundaries that break up patterns or span multiple patterns (outlined in \textcolor{red}{red}).}
        \label{fig:color-simple}
\end{minipage}\hfill
    \begin{minipage}{0.54\linewidth}
\captionof{table}{\footnotesize Effect of \ours~on the \emph{Simple Colors} and \emph{Conditional Colors} datasets (5 seeds). All approaches achieve similar values for the likelihood objective, but \ours better recovers the correct boundaries.}
            \vspace{-2mm}
\center\small
            \resizebox{\linewidth}{!}{\begin{tabular}{ccccc}
                \toprule
                &\multicolumn{2}{c}{\textit{Simple Colors}}& \multicolumn{2}{c}{\textit{Conditional Colors}}\\
                 & \textbf{VTA} & \textbf{\ours} & \textbf{VTA} & \textbf{\ours}  \\
                 \cmidrule(lr){2-2} \cmidrule(lr){3-3} \cmidrule(lr){4-4} \cmidrule{5-5}
$\gL_{\text{ELBO}}$  & $2868 {\scriptstyle\, \pm \, 43}$ & $2838 {\scriptstyle\, \pm \, 19}$ & $2832 {\scriptstyle\, \pm \, 7.7}$ & $2827 {\scriptstyle\, \pm \, 1.9}$\\
\midrule
                Precision & $0.87 {\scriptstyle\, \pm \, 0.19}$ & $\mathbf{0.99} {\scriptstyle\, \pm \, 0.01}$ & $0.84 {\scriptstyle\, \pm \, 0.22}$ & $\mathbf{0.99} {\scriptstyle\, \pm \, 0.01}$\\
                Recall & $0.79 {\scriptstyle\, \pm \, 0.13}$ & $\mathbf{0.85} {\scriptstyle\, \pm \, 0.03}$ & $\mathbf{0.82} {\scriptstyle\, \pm \, 0.16}$ & $\mathbf{0.83} {\scriptstyle\, \pm \, 0.06}$ \\
                F1 & $0.82 {\scriptstyle\, \pm \, 0.13}$ & $\mathbf{0.91} {\scriptstyle\, \pm \, 0.02}$ & $0.83 {\scriptstyle\, \pm \, 0.19}$ & $\mathbf{0.90} {\scriptstyle\, \pm \, 0.03}$ \\
Code Length & $7.58 {\scriptstyle\, \pm \, 1.3}$ & $\mathbf{6.34} {\scriptstyle\, \pm \, 0.75}$ & $9.17 {\scriptstyle\, \pm \, 1.76}$ & $\mathbf{6.83} {\scriptstyle\, \pm \, 0.51}$ \\
                \bottomrule
            \end{tabular}
            }
\label{tab:synthetic-results}

\end{minipage}
    \vspace{-3mm}
\end{figure}

\paragraph{Dataset.}
We consider two simple datasets \emph{Simple Colors} and \emph{Conditional Colors}, which consist of sequences of $32 \times 32$ monochromatic frames, with repeating underlying patterns.
\emph{Simple Colors} consists of 3 patterns: 3 consecutive yellow frames, 3 consecutive blue frames and 3 consecutive green frames.
The dataset is generated by sampling these patterns with probability 0.4, 0.4, and 0.2 respectively and concatenating the results.
Each pattern is sampled independent of history.
The dataset is then divided into sequences of 6 patterns, equal to $6 \times 3 = 18$ frames.

In contrast to \emph{Simple Colors}, \emph{Conditional Colors} tests learning patterns that depend on history.
It consists of 4 patterns: the 3 patterns from \emph{Simple Colors} with an additional pattern of 3 consecutive purple frames.
The dataset is generated by uniformly sampling the patterns from \emph{Simple Colors}.
To make the current timestep pattern dependent on the previous timestep, the yellow frames are re-colored to purple, if the previous pattern was blue or yellow.
As in \emph{Simple Colors}, the dataset is divided into sequences of 6 patterns.

In both datasets, the optimal encoding strategy is to learn 3 subsequence descriptors (i.e., skills), one for each pattern in \emph{Simple Colors}.
In \emph{Conditional Colors}, the descriptor corresponding to yellow in \emph{Simple Colors} either outputs yellow or purple, depending on the history.
This encoding strategy achieves an expected code length of $6.32$ nats in \emph{Simple Colors} and $6.59$ nats in \emph{Conditional Colors}.

\paragraph{Results.}
We visualize the learned descriptors and boundaries in Figure \ref{fig:color-simple}.
\ours successfully segments the sequences into the patterns and assigns a consistent descriptor $\rvz$ to each pattern.
In contrast, despite the simple structure of this problem, VTA learns subsequences that span over multiple patterns or last for fewer timesteps than a single pattern.

Quantitatively, we measure the (1) precision, recall, and F1 scores of the boundary prediction, (2) the ELBO of the maximum likelihood objective and (3) the average code length $\mathcal{L}_\text{CL}$ in Table \ref{tab:synthetic-results}.
While both VTA and \ours achieve approximately the same negated ELBO (lower is better), \ours recovers the correct boundaries with much higher precision, recall, and F1 scores.
This illustrates that the underspecification problem can occur even in simple sequential data.
Additionally, we find that \ours achieves an encoding cost close to the optimal value.
One interesting, though rare, failure mode is that \ours sometimes even achieves a lower encoding cost than the optimal value by over-weighting the compression term and imperfectly reconstructing the data. In Appendix~\ref{app:sensitivity}, we conduct a ablation study on the weights of $\gL_{\text{CL}}$ in optimization.

\section{Experiments}\label{sec:experiments}
\begin{wrapfigure}[14]{r}{0.36\textwidth}
\vspace{-8mm}
\begin{center}
     \includegraphics[width=0.35\textwidth]{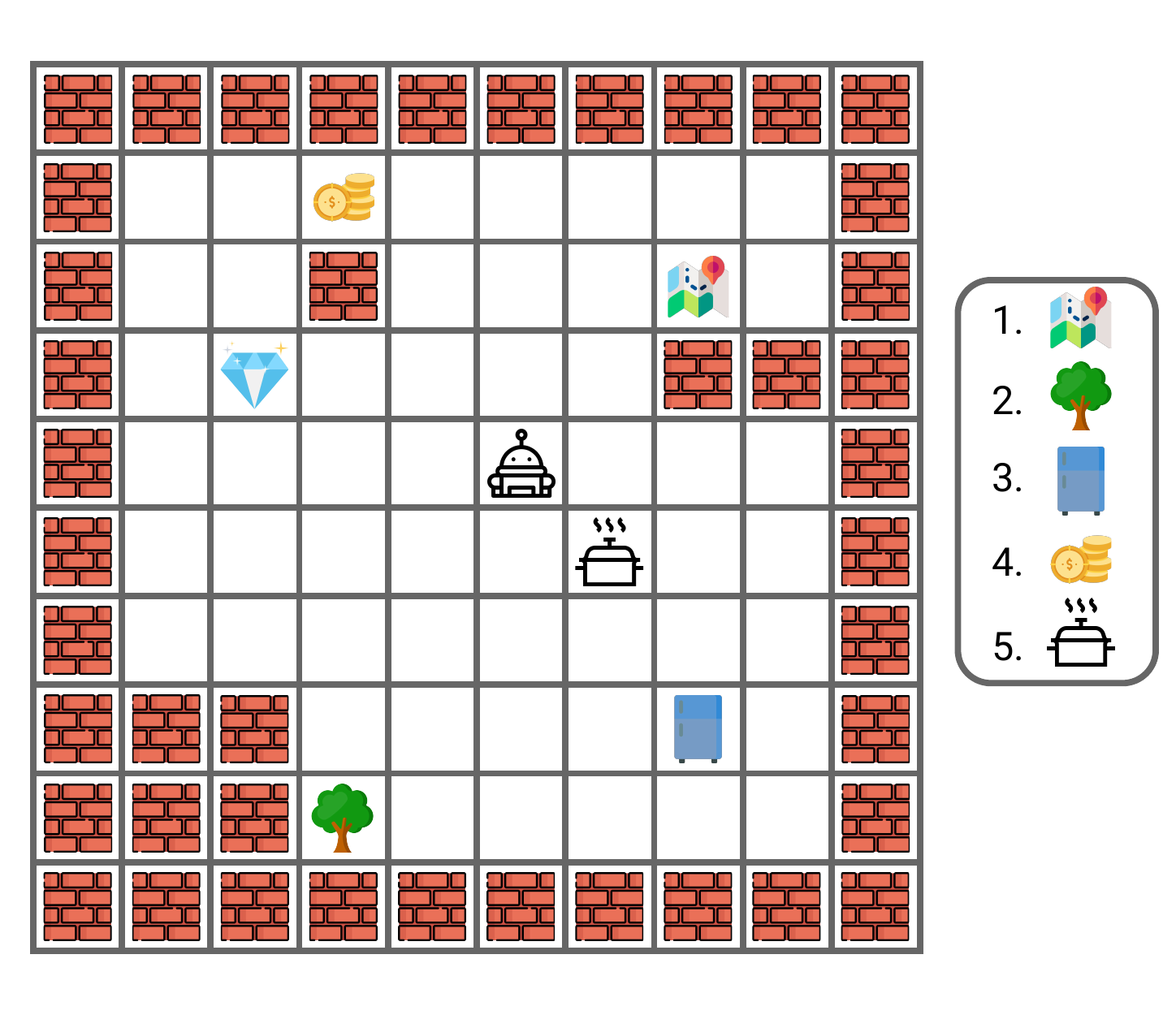}
\caption{Multi-task grid world instance from \citet{kipf2019compile}. The agent must pick up a sequence of objects in a specified order.}
     \label{fig:grid_world}
\end{center}
\end{wrapfigure}

We aim to answer three questions:
(1) Does \ours learn semantically meaningful skills?
(2) Do skills learned with \ours accelerate learning for new tasks?
(3) Can \ours scale to high-dim pixel observations?
To answer these questions, we learn skills from demonstrations and use these skills to learn new tasks, first on a 2D multi-task domain, and later on a 3D variant.

\paragraph{Multi-task domain.}
We consider the multi-task $10 \times 10$ grid world introduced by \citet{kipf2019compile}, a representative skill learning approach (Figure \ref{fig:grid_world}).
This domain features a challenging control problem that requires the agent to collect up to 5 objects, and we later add a perception challenge to it with the 3D variant.
We use our own custom implementation, since the code from \citet{kipf2019compile} is not publicly available.
In each task, 6 objects are randomly sampled from a set of $N_\text{obj} = 10$ different objects and placed at random positions in the grid.
Additionally, impassable walls are randomly added to the grid, and the agent is also placed at a random initial grid cell.
Each task also includes an instruction list of $N_\text{pick} = 3$ or $5$ objects that the agent must pick up in order.

Within each task, the agent's actions are to move up, down, left, right, or to pick up an item at the agent's grid cell.
The state consists of two parts: (1) a $10 \times 10 \times (N_\text{obj} + 2)$ grid observation, indicating if the agent, a wall, or any of the $N_\text{obj}$ different object types is present at each of the $10 \times 10$ grid cells; (2) the instruction corresponding to the next object to pick up, encoded as a one-hot integer.
Following \citet{kipf2019compile}, we consider two reward variants:
In the \emph{sparse reward} variant, the agent only receives $+1$ reward after picking up \emph{all} $N_\text{pick}$ objects in the correct order.
In the \emph{dense reward} variant, the agent receives $+1$ reward after picking up each specified object in the correct order.
Our dense reward variant is slightly harder than the variant in \citet{kipf2019compile}, which gives the agent $+1$ reward for picking up objects in \emph{any} order.
The agent receives $0$ reward in all other time steps.
The episode ends when the agent has picked up all the objects in the correct order, or after $50$ timesteps.

\paragraph{Pre-collected experience.} We follow the setting in \citet{kipf2019compile}.
We set the pre-collected experience to be $2000$ demonstrations generated via breadth-first search on randomly generated tasks with only $N_\text{pick} = 3$ and test if the agent can generalize to $N_\text{pick} = 5$ when learning a new task.
These demonstrations are not labeled with rewards and also \textit{do not} contain the instruction list observations.

\paragraph{Points of comparison.} To study our compression objective, we compare with two representatives of learning skills via the maximum likelihood objective:
\begin{itemize}[itemsep=0pt, topsep=0pt, leftmargin=10mm]
    \item VTA \citep{kim2019variational}, modified to handle interaction data as in Section \ref{sec:vta_actions}.
    \item Discovery of deep options (DDO) \citep{fox2017multi}.
    We implement DDO's maximum likelihood objective and graphical model on top of VTA's variational inference optimization procedure, instead of expectation-gradients used in the original paper.
\end{itemize}
For fairness, we also compare with variants of these that implement the minimum skill length constraint from Section \ref{sec:practical}.
CompILE~\citep{kipf2019compile} is another approach that learns skill by maximizing likelihood and was introduced in the same paper as this grid world.
Because the implementation is unavailable, we do not compare with it.
However, we note that CompILE requires additional supervision that \ours, VTA, and DDO do not: namely, it requires knowing how many skills each demonstration is composed of.
This supervision can be challenging to obtain without already knowing what the skills should be.

Since latent variable models are prone to local optima~\citep{kumar2010self}, it is common to learn such models with multiple restarts~\citep{murphy2012machine}.
We therefore run each method with 3 random initializations
and pick the best model according to the compression objective $\mathcal{L}_\text{CL}$ for \ours and according to the ELBO of the maximum likelihood objective $\mathcal{L}_\text{ELBO}$ for the others.
Notably, this does not require any additional demonstrations or experience.

We begin by analyzing the skills learned from the pre-collected experience before analyzing the utility of these skills for learning downstream tasks in the next sections.

\subsection{Analyzing the Learned Skills}\label{sec:analyzing_skills}

\begin{wraptable}[11]{r}{0.55\textwidth}
    \vspace{-4mm}
    \caption{\footnotesize Precision and recall of outputting the boundaries between picking up different objects. Only \ours learns skills that move to and pick up an object.}
    \vspace{-3mm}
    \center\small
    \begin{tabular}{cccc}
        \toprule
        & Precision & Recall & F1 \\ \midrule
DDO \citep{fox2017multi} & 0.19 & \textbf{1} & 0.32 \\
        DDO + min. skill length & 0.26 & 0.53 & 0.35 \\
        VTA \citep{kim2019variational} & 0.19 & \textbf{0.99} & 0.32 \\
        VTA + min. skill length & 0.27 & 0.53 & 0.36\\
        \ours (ours) & \textbf{0.90} & 0.94 & \textbf{0.92} \\
        \bottomrule
    \end{tabular}
    \label{tab:grid_prec_recall}
\end{wraptable}
We analyze the learned skills by comparing them to a natural decomposition that partitions the demonstration in $N_\text{pick}$ sequences of moving to and picking up an object.
Specifically, we measure the precision and recall of each method in outputting the correct boundaries of these $N_\text{pick}$ sequences.

\begin{figure*}[t]
\begin{center}
\vspace{-7mm}
    \centerline{\includegraphics[width=\linewidth]{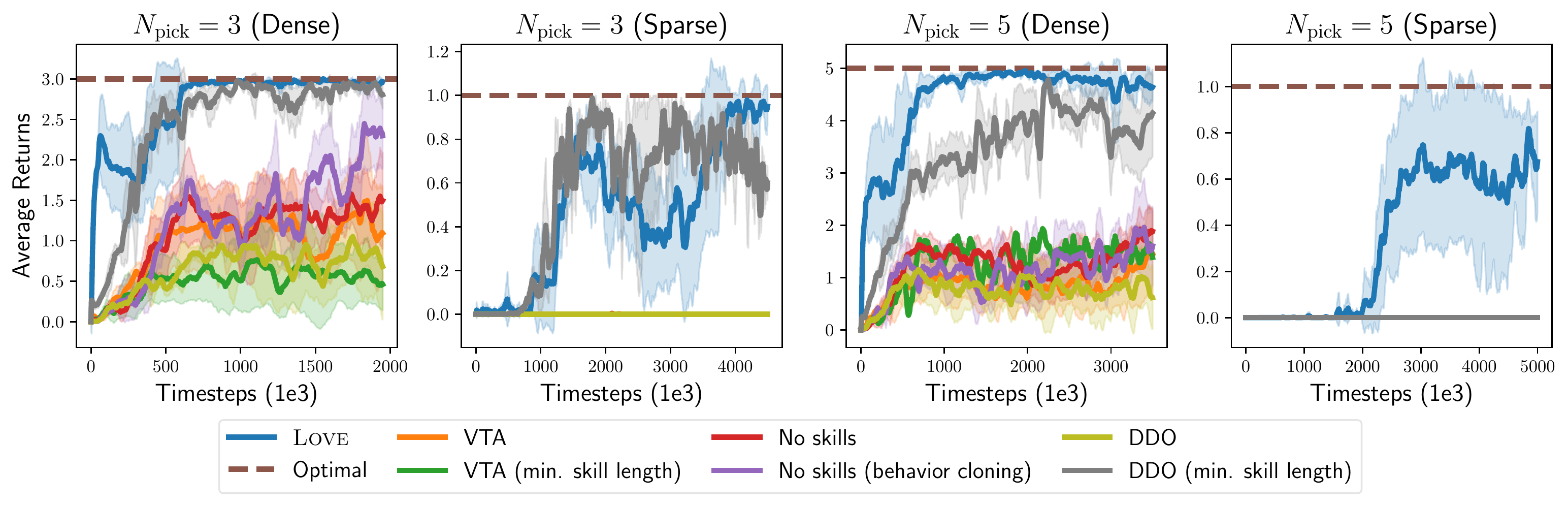}}
\vspace{-1mm}
\caption{\textbf{Sample efficient learning.} We plot returns vs. timesteps of environment interactions for 4 settings in the grid world with 1-stddev error bars (5 seeds). Only \ours achieves high returns across all 4 settings.}
\label{fig:grid_returns}
\end{center}
\vspace{-5mm}
\end{figure*}

We find that only \ours recovers the correct boundaries with both high precision and recall (Table \ref{tab:grid_prec_recall}).
Visually examining the skills learned with \ours shows that each skill moves to an object and picks it up.
We visualize \ours's skills in Appendix \ref{app:visualization}.
In contrast, maximizing likelihood with either DDO or VTA learns degenerate skills that output only a single action, leading to high recall, but low precision.
Adding the minimum skill length constraint to these approaches prevents such degenerate solutions, but still does not lead to the correct boundaries.
Skills learned by VTA with the minimum skill length constraint exhibit no apparent structure, but skills learned by DDO with the minimum skill length constraint appear to move toward objects, though they terminate at seemingly random places.
These results further illustrate the underspecification problem.
All approaches successfully learn to reconstruct the demonstrations and achieve high likelihood, but only \ours and DDO with the minimum skill length constraint learn skills that appear to be semantically meaningful.

\subsection{Learning New Tasks}
Next, we evaluate the utility of the skills for learning new tasks.
To evaluate the skills, we sample a new task in each of 4 settings: with sparse or dense rewards and with $N_\text{pick} = 3$ or $N_\text{pick} = 5$.
Then we learn the new tasks following the procedure in Section \ref{sec:hrl}:
we augment the action space with the skills and train an agent over the augmented action space.
To understand the impact of skills, we also compare with a low-level policy that learns these tasks using only the original action space, and the same low-level policy that incorporates the demonstrations by pre-training with behavioral cloning.
We parametrize the policy for all approaches with dueling double deep Q-networks \citep{mnih2015human, wang2016dueling, van2016deep} with $\epsilon$-greedy exploration.
We report the returns of evaluation episodes with $\epsilon = 0$, which are run every 10 episodes.
We report returns averaged over 5 seeds.
See Appendix \ref{app:hrl-hparam} for full details.

\paragraph{Results.}
Overall, \ours learns new tasks across all 4 settings comparably to or faster than both skill methods based on maximizing likelihood and incorporating demonstrations via behavior cloning (Figure \ref{fig:grid_returns}).
More specifically, when $N_\text{pick} = 3$, DDO with the min. skill length constraint performs comparably to \ours, despite not segmenting the demonstrations into sequences of moving to and picking up objects.
We observe that while a single DDO skill does not move to and pick up an object, the skills consistently move toward objects, which likely helps with exploration.
Imitating the demonstrations with behavior cloning also accelerates learning when $N_\text{pick} = 3$ with dense rewards, though not as much as \ours or DDO, and yields insignificant benefit in all other settings.
Skill learning with VTA yields no benefit over no skill learning at all.

Recall that $N_\text{pick} = 3$ in all of the demonstrations.
\ours learns new tasks in the generalization setting of $N_\text{pick} = 5$ much faster than the other methods.
With dense rewards, DDO (min. skill length) also eventually solves the task, but requires over $8\times$ more timesteps.
The other approaches learn to pick up some objects, but never achieve optimal returns of 5.
With sparse rewards, \ours is the only approach that achieves high returns, while all other approaches achieve 0 returns.
This likely occurs because this setting creates a challenging exploration problem, so exploring with low-level actions or skills learned with VTA or DDO rarely achieves reward.
By contrast, exploring with skills learned with \ours achieves rewards much more frequently, which enables \ours to ultimately solve the task.

\subsection{Scaling to High-dimensional Observations}

Prior works in the offline multi-task setting have learned skills from low-dimensional states and then transferred them to high-dim pixel observations~\citep{rao2021learning}, or have learned hierarchical models from pixel observations~\citep{kim2019variational, fox2017multi}.
However, to the best of our knowledge, prior works in this setting have not directly learned skills from high-dim pixel observations, only from low-dimensional states.
Hence, in this section, we test if \ours can scale up to high-dim pixel observations by considering a challenging 3D image-based variant of the multi-task domain above, illustrated in Figure~\ref{fig:3dworld}.
The goal is still to pick up $N_{\text{pick}}$ objects in the correct order, where objects are blocks of different colors.
Each task includes 3 objects randomly sampled from a set of $N_\text{obj} = 6$ different colors.
The observations are $400 \times 60 \times 3$ egocentric panoramic RGB arrays.
The actions are to move forward / backward, turn left / right, and pick up.

\begin{figure}
    \begin{minipage}{0.35\linewidth}
    \centerline{\includegraphics[width=0.95\columnwidth]{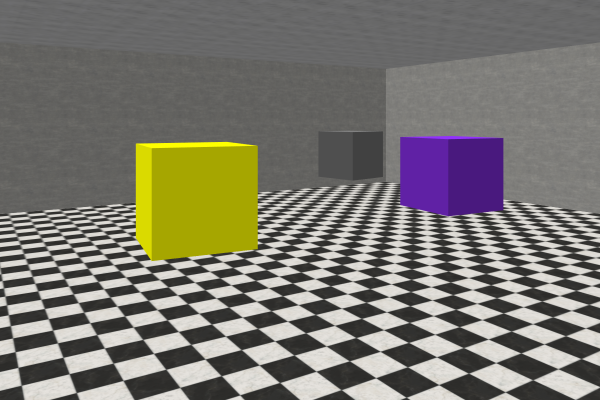}}
    \vspace{2.2mm}
    \caption{3D visual multi-task domain.}
    \label{fig:3dworld}
\end{minipage}\hfill
    \begin{minipage}{0.6\linewidth}
    \vspace{2mm}
    \centerline{\includegraphics[width=0.79\columnwidth]{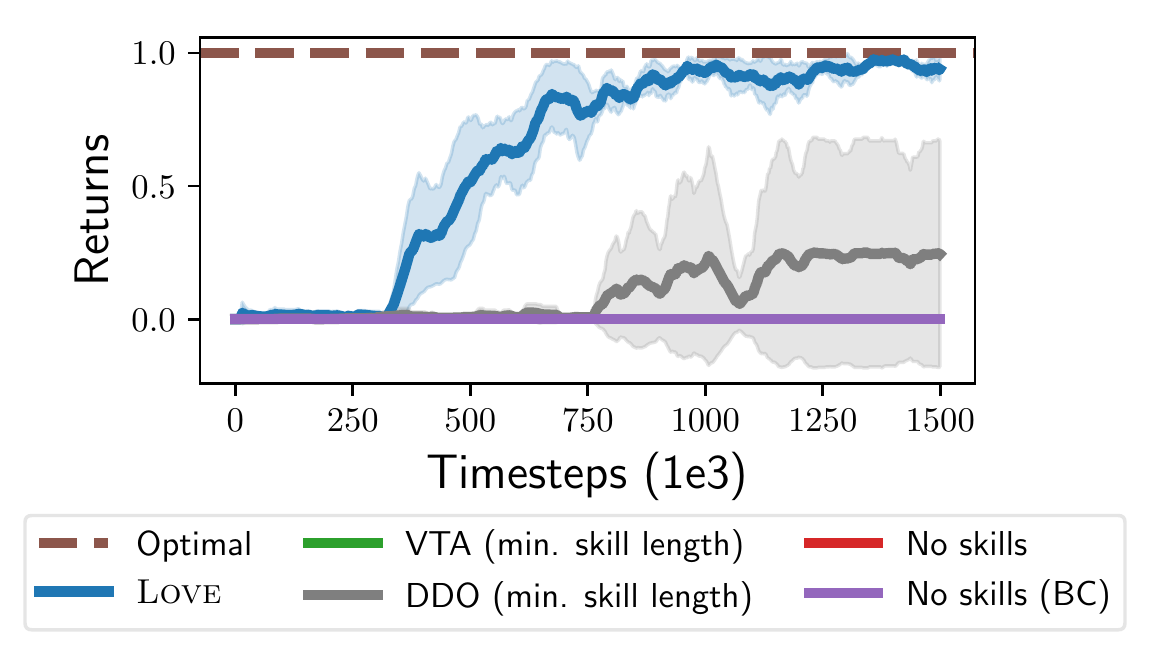}}
    \vspace{-2.2mm}
    \caption{Returns on the 3D visual multi-task domain with 1-stddev error bars (5 seeds). Only \ours achieves near-optimal returns.}
    \label{fig:3d_results}
\end{minipage}
\end{figure}

We consider only the harder \emph{sparse reward} and generalization setting, where the agent must pick up $N_{\text{pick}} = 4$ objects in the correct order, after receiving 2,000 pre-collected demonstrations of the approximate shortest path of picking up $N_{\text{pick}} = 2$ objects in the correct order.
We use the same hyperparameters as in multi-task domain and only change the observation encoder and action decoder (full details in Appendix~\ref{app:3d_arch}).
Figure~\ref{fig:3d_results} shows the results.
Again, \ours learns skills that each navigate to and pick up an object, which quickly achieves optimal returns, indicating its ability to scale to high-dim observations.
In contrast, the other approaches perform much worse on pixel observations than the previous low-dim observations.
DDO also makes progress on the task, but requires far more samples, as it again learns skills that only operate for a few timesteps, though they move toward the objects.
All other approaches fail to learn at all, including the variants of VTA and DDO without the min. skill length, which are not plotted.

\section{Conclusion}\label{sec:conclusion}
We started by highlighting the underspecification problem: maximizing likelihood does not necessarily yield skills that are useful for learning new tasks.
To overcome this problem, we drew a connection between skill learning and compression and proposed a new objective for skill learning via compression and a differentiable model for optimizing our objective.
Empirically, we found that the underspecification problem occurs even on simple tasks and learning skills with our objective allows learning new tasks much faster than learning skills by maximizing likelihood.

Still, important future work remains. \ours applies when there are useful and consistent structures that can be extracted from multiple trajectories. This is often present in multi-task demonstrations, which solve related tasks in similar ways. However, an open challenge for adapting to general offline data like D4RL~\citep{fu2020d4rl}, is to ensure that the learned skills do not overfit to potentially noisy or unhelpful behaviors often present in offline data.
In addition, we showed a connection between skill learning and compression by minimizing the description length of a crude two-part code.
However, we only proposed a way to optimize the message length term and not the model length.
Completing this connection by accounting for model length may therefore be an promising direction for future work.

\paragraph{Reproducibility.} 
Our code is publicly available at \url{https://github.com/yidingjiang/love}.

\paragraph{Acknowledgements.}
{
We would like to thank Karol Hausman, Abhishek Gupta, Archit Sharma, Sergey Levine, Annie Xie, Zhe Dong, Samuel Sokota for valuable discussions during the course of this work, and Victor Akinwande, Christina Baek, Swaminathan Gurumurthy for comments on the draft.
We would also like to thank the anonymous reviewers for their valuable feedback.
Last, but certainly not least, YJ and EZL thank the beautiful beach of Kona for initiating this project.

YJ is supported by funding
from the Bosch Center for Artificial Intelligence.
EZL is supported by a National Science Foundation Graduate Research Fellowship under Grant No. DGE-1656518.
CF is a Fellow in the CIFAR Learning in Machines and Brains Program.
This work was also supported in part by the ONR grant N00014-21-1-2685, Google, and Intel.
Icons in this work were made by ThoseIcons, FreePik, VectorsMarket, from FlatIcon.

}

\begin{center}
    \textit{\ours is a many-splendored thing. \ours lifts us up where we belong. All you need is \ours.} \\
    --- Moulin Rouge (Elephant Love Song Medley)
\end{center}

\bibliographystyle{plainnat}
\bibliography{example_paper}

\newpage

\newpage
\appendix
\onecolumn

\section{Derivations}

\subsection{$p^{\star}_\rvz$ is a valid density}
\label{app:valid_density}
\begin{proposition}
$p^{\star}_\rvz(\xi)$ is a proper probability density function.
\end{proposition}
\begin{proof}
Recall that
$$p^{\star}_\rvz(\xi) = \frac{\E_{\substack{\mathbf{\tau}_{1:T},\rvm_{1:T},\\\rvz_{1:T}}}\left[\sum_{t=1}^T \delta(\rvz_t = \xi)\rvm_t\right]}{\E_{\substack{\mathbf{\tau}_{1:T},\rvm_{1:T}}}\left[\sum_{t=1}^T \rvm_t\right]}.$$  
A function is a probability density function if it is non-negative, continuous, and integrates to 1. By construction, $p^{\star}_\rvz(\xi)$ is non-negative since $\rvm_t \geq 0$ and $\delta(\rvz_t = \xi) \geq 0$. Further, $\delta(\rvz_t = \xi)$ is continuous in $\xi$ so the numerator $\E_{\substack{\mathbf{\tau}_{1:T},\rvm_{1:T},\\\rvz_{1:T}}}\left[\sum_{t=1}^T \delta(\rvz_t = \xi)\rvm_t\right]$, an expectation containing $\delta(\rvz_t = \xi)$, is also continuous in $\xi$. Finally,
\begin{align*}
    \int_\xi p^{\star}_\rvz(\xi) \, d\xi  &= \frac{1}{\E_{\substack{\mathbf{\tau}_{1:T},\rvm_{1:T}}}\left[\sum_{t=1}^T \rvm_t\right]} \int_\xi \E_{\substack{\mathbf{\tau}_{1:T},\rvm_{1:T}, \rvz_{1:T}}}\left[\sum_{t=1}^T \delta(\rvz_t = \xi)\rvm_t\right] d\xi     \\
    &=\frac{1}{\E_{\substack{\mathbf{\tau}_{1:T},\rvm_{1:T}}}\left[\sum_{t=1}^T \rvm_t\right]}  \E_{\substack{\mathbf{\tau}_{1:T},\rvm_{1:T},\rvz_{1:T}}}\left[\sum_{t=1}^T \left(\int_\xi\delta(\rvz_t = \xi)d\xi\right)  \rvm_t\right]  &&\text{(Fubini's theorem)} \\
    &=\frac{1}{\E_{\substack{\mathbf{\tau}_{1:T},\rvm_{1:T}}}\left[\sum_{t=1}^T \rvm_t\right]}  \E_{\substack{\mathbf{\tau}_{1:T},\rvm_{1:T},\rvz_{1:T}}}\left[\sum_{t=1}^T 1 \cdot \rvm_t\right] \\
    &=1.
\end{align*}
Therefore, $p^{\star}_\rvz(\xi)$ satisfies all 3 properties of a probability density function. 
\end{proof}
\paragraph{Remarks.} This result generalizes to the case for discrete $\rvz$ by replacing the delta with an indicator function and integration with summation.

\subsection{Two expressions of $\gL_{\text{CL}}$}
\label{app:rewrite}
\begin{proposition}
The expected coding length of a trajectory is equal to the expected number skills multiplied by the marginal entropy:
\begin{align*}
\gL_{\text{CL}}(\mathbf{\theta}) = -\E_{\substack{\mathbf{\tau}_{1:T}, \rvm_{1:T},\\\rvz_{1:T}}}\left[\sum_{t=1}^T \log p_{\rvz}^{\star}(\rvz_t) \rvm_t\right] = n_\text{s}\gH_{p_{\rvz}^{\star}}[\rvz].
\end{align*}
\end{proposition}

\begin{proof}
\begingroup
\allowdisplaybreaks
\begin{align*}
    &\quad -\E_{\substack{\mathbf{\tau}_{1:T}, \rvm_{1:T}, \rvz_{1:T}}}\left[\sum_{t=1}^T \log p_{\rvz}^{\star}(\rvz_t) \rvm_t\right] \\
    &= -\E_{\substack{\mathbf{\tau}_{1:T}, \rvm_{1:T},  \rvz_{1:T}}}\left[\sum_{t=1}^T \left(\int_{\xi} \delta(\rvz_t = \xi)  \log p_{\rvz}^{\star}(\xi) d\xi \right)  \rvm_t \right] \\
    &=-\E_{\substack{\mathbf{\tau}_{1:T}, \rvm_{1:T},\rvz_{1:T}}}\left[ \int_{\xi} \sum_{t=1}^T \rvm_t \delta(\rvz_t = \xi)  \log p_{\rvz}^{\star}(\xi) d\xi   \right] &&\text{(Fubini's theorem)} \\
    &=- \int_{\xi} \E_{\substack{\mathbf{\tau}_{1:T},\rvm_{1:T},\rvz_{1:T}}}\left[\sum_{t=1}^T \rvm_t \delta(\rvz_t = \xi)\right]  \log p_{\rvz}^{\star}(\xi) d\xi&& \text{(Linearity of expectation)}   \\
    &=- \E_{\substack{\mathbf{\tau}_{1:T},\rvm_{1:T}}}\left[\sum_{t=1}^T \rvm_t\right]\int_{\xi} \frac{\E_{\substack{\mathbf{\tau}_{1:T},\rvm_{1:T},\rvz_{1:T}}}\left[\sum_{t=1}^T \rvm_t \delta(\rvz_t = \xi)\right]}{\E_{\substack{\mathbf{\tau}_{1:T},\rvm_{1:T}}}\left[\sum_{t=1}^T  \rvm_t\right]} \log p_{\rvz}^{\star}(\xi) d\xi  \\
    &= \underbrace{\E_{\substack{\mathbf{\tau}_{1:T},\rvm_{1:T},\rvz_{1:T}}}\left[\sum_{t=1}^T \rvm_t\right]}_{n_\text{s}} \,\, \underbrace{\int_{\xi} -p_{\rvz}^{\star}(\xi) \log p_{\rvz}^{\star}(\xi) d\xi}_{\gH_{p_{\rvz}^{\star}}[\rvz]} \\
    &= n_\text{s}\gH_{p_{\rvz}^{\star}}[\rvz]
\end{align*}
\endgroup
\end{proof}
\paragraph{Remarks.} This result generalizes to the case for discrete $\rvz$ by replacing the delta with an indicator function and integration with summation. Further, this rewriting not only offers an insightful interpretation of the objective but also has \textit{different} finite-sample gradient compared to the original form.
We discuss its implication further in Section \ref{sec:practical}.

\subsection{Connection between \ours and Variational Inference}
\label{app:prior}
Is the compression objective a prior in the framework of variational inference (VI)? The answer to this question is surprisingly nuanced. A bridge that connects VI and source coding is the bits-back coding argument~\citep{hinton1994autoencoders, honkela2004variational}. Under this scheme, the additional cost of communicating a message is equal to the KL divergence between the approximate posterior and the chosen prior distribution. This is because the posterior often does not match the prior and the sender must use more bits to send the message compared to using the prior exactly. In modern deep latent variable models such as VAE~\citep{kingma2013auto}, the prior is in general chosen to be an isotropic Gaussian distribution or a sequence of isotropic Gaussian distributions that decomposes over the time steps. The bit-back coding argument shows that VI as a coding scheme is asymptotically optimal (in the limit of sending large numbers of messages), but it does not immediately guarantee good \textit{representation learning}. As a thought experiment, imagine we choose the dimension of the latent variable in a VAE to be equal to the dimension of the data itself. It is clear that the model is not incentivized to perform good representation learning, but the bit-backing coding scheme still guarantees asymptotic optimality. Instead, good representation can only emerge \textit{if} the prior is chosen properly.

Contrary to this prevailing design choice, our latent variables $\rvz_{1:T}$ and $\rvm_{1:T}$ are not independent from each other and the timesteps are not independent from each other. In fact, the interaction between the latent variables is crucial for compressing the trajectory optimally in our framework. Therefore, our objective takes into account the \textit{global} structure of the latent variables. Note that $\gL_\text{CL}$ also has the natural interpretation of measuring the average number of bits that is required to send a trajectory under the coding scheme of this work. As both quantities measure the cost of the communication, the philosophical parallel between KL divergence and $\gL_\text{CL}$ is likely not a coincidence. In practice, we observe that $\gL_\text{CL}$ also has comparable regularization effect on the latent code, and in some cases it is possible to learn a good model \textit{without} any KL divergence if $\gL_\text{CL}$ is present. From this perspective, our compression objective is indeed a ``prior'' insofar as it specifies the desired properties of the posterior. Another interpretation of \ours is that it is optimizing the prior \textit{directly}.

Nonetheless, it is an open question whether $\gL_\text{CL}$ has a precise counterpart in Bayesian inference, that is, it is not immediately obvious if it can be written as the KL divergence between the posterior $q(\rvz_{1:T}, \rvm_{1:T} \mid \rvx_{1:T})$ and some prior distribution $p(\rvz_{1:T}, \rvm_{1:T})$. One complication, for example, is the fact that the log density of $\rvm_{1:T}$ does not participate in the computation of $\gL_\text{CL}$. It may be possible to construct some form of hierarchical priors that enables a fully Bayesian interpretation of $\gL_\text{CL}$, but it is also well-known that the MDL framework can accommodate codes that are not Bayesian (e.g., the Shtarkov normalized maximum likelihood code~\citep{shtarkov1987univresal}). We do not offer a definitive answer to this question in this work, but the connection between Bayesian inference and $\gL_\text{CL}$ is certainly an interesting direction for future works.

\section{Evidence Lower Bound}
\label{app:elbo}
\subsection{VTA}
We reproduce the data generating process of \citet{kim2019variational} here:
\begin{align*}
    p(\rvx_{1:T}, \rvz_{1:T}, \rvs_{1:T}, \rvm_{1:T}) = \prod_{t=1}^T \,\, p(\rvx_t \mid \rvs_t) p(m_t \mid \rvs_t) p(\rvs_t \mid \rvs_{1:t}, \rvz_t, m_{t-1}) p(\rvz_t \mid \rvz_{1:t}, m_{t-1}).
\end{align*}
VTA~\citep{kim2019variational} shows that the ELBO can be written as:
\begin{align*}
\gL_{\text{ELBO}}(\boldsymbol{\theta}, \boldsymbol{\phi})  \approx \sum_{t=1}^T \quad &\log p_{\boldsymbol{\phi}}(\rvx_t \mid \rvs_t) + D_{\text{KL}}\left(q_{\boldsymbol{\theta}}(m_t \mid \rvx_{1:T}) \, \| \, p_{\boldsymbol{\phi}}(m_t \mid \rvs_t) \right)\\
 +\,\,& D_{\text{KL}}\left(q_{\boldsymbol{\theta}}(\rvz_t \mid \rvs_t, \rvm_{1:T}, \rvx_{1:T}) \, \| \, p_{\boldsymbol{\phi}}(\rvz_t \mid \rvs_{t-1}, m_{t-1}) \right) \\
+\,\,& D_{\text{KL}}\left(q_{\boldsymbol{\theta}}(\rvs_t \mid \rvs_{t-1}, \rvm_{1:T}, \rvx_{1:T}) \, \| \, p_{\boldsymbol{\phi}}(\rvs_t \mid \rvs_{t-1}, \rvz_t,  m_{t-1}) \right).
\end{align*}
Like most ELBO, this ELBO can be deconstructed into a reconstruction term:
\begin{align*}
    \gL_{\text{rec}}(\boldsymbol{\theta}, \boldsymbol{\phi}) = \sum_{t=1}^T &\log p_{\boldsymbol{\phi}}(\rvx_t \mid \rvs_t),
\end{align*}
and the remaining KL term $\gL_{\text{KL}}(\boldsymbol{\theta}, \boldsymbol{\phi})$.
\subsection{\ours}
We reproduce the data generating process of \ours here. Instead of using $\boldsymbol{\phi}$ for the parameters of priors and decoder, we will merge all parameters into $\boldsymbol{\theta}$:
\begin{align*}
    p(\rva_{1:T}, \rvz_{1:T}, \rvs_{1:T}, \rvm_{1:T}\mid \rvx_{1:T}) = \prod_{t=1}^T \,\, p(\rva_t \mid \rvs_t) p(m_t \mid \rvs_t) p(\rvs_t \mid \rvx_{t}, \rvz_t, m_{t-1}) p(\rvz_t \mid \rvx_{t}, \rvz_{1:t}, m_{t-1}).
\end{align*}

For \ours, the ELBO is re-written as:
\begin{align}\label{eqn:rl_elbo}
\gL_{\text{ELBO}}(\boldsymbol{\theta})  \approx \sum_{t=1}^T \quad &\log p_{\boldsymbol{\theta}}(\rva_t \mid \rvs_t) \\+\,\,& D_{\text{KL}}\left(q_{\boldsymbol{\theta}}(m_t \mid \rvx_{1:t}) \, \| \, p_{\boldsymbol{\theta}}(m_t \mid \rvs_t) \right)
 + D_{\text{KL}}\left(q_{\boldsymbol{\theta}}(z_t \mid m_t, \rvx_{1:T}, \rva_{1:T}) \, \| \, p(z_t) \right)\nonumber \\
+\,\,& D_{\text{KL}}\left(q_{\mathbf{\theta}}(\rvs_t \mid z_t, \rvx_t) \, \| \, p_{\boldsymbol{\theta}}(\rvs_t \mid \rvs_{t-1}, m_{t-1}) \right). \nonumber
\end{align}
Since $z_t$ is a categorical distribution, we chose the uniform prior $p(z_t)$ similar to \citet{oord2017neural}.
Similar to VTA, the loss is also decomposed into a reconstruction term and a KL term.

\section{Sensitivity to Compression Objective Weighting}
\label{app:sensitivity}
Here we study how sensitive \ours is to the choice of $\lambda$, the coefficient of $\gL_{\text{CL}}$. We tested 5 values of $\lambda \in [0.01, 1]$ on \emph{Simple} and \emph{Cond. Colors} and applied dual gradient descent to automatically tune $\lambda$. Stopping point is selected based on the shortest code length achieved.
\begin{table}[H]
\centering
\caption{The effect of different values of $\lambda$ and dual GD on F1 scores of \textit{Simple Colors} and \textit{Conditional Colors} over 5 seeds.}
\small
\begin{tabular}{cccccccc}
\toprule
& \!$\lambda \! =\! 0.01$\! & \!\!$\lambda \!=\! 0.03$\! & \!\!\!$\lambda \!=\! 0.1$\!\!\! & \!\!\!\!\!$\lambda \!=\! 0.3$\!\!\!\! & \!\!\!\!$\lambda \!=\! 1.0$\!\!\! & $\!\!\lambda$ (dual GD)\!\! & \!\!VTA\!\! \\
\cmidrule(lr){2-2} \cmidrule(lr){3-3} \cmidrule(lr){4-4} \cmidrule(lr){5-5} \cmidrule(lr){6-6} \cmidrule(lr){7-7} \cmidrule(lr){8-8}
\!\!Simple (F1)\!\!\! & $.67\pm.27$ & \!\!$.80 \pm .21$ \! & $.91 \pm .02$ & $.95\pm.06$ & $.74\pm .08$ & $.99\pm.00$ & $.82\pm.13$ \\
\!\!Cond. (F1)\!\!\! & $.84\pm .06$ & \!\!$.92 \pm .02$\! & $.90 \pm .03$ & $.82 \pm .08$ & $.87 \pm .08$ & $.99 \pm .00$& $.83 \pm .19$ \\
\bottomrule
\end{tabular}
\vspace{1mm}
\label{tab:lambda_ablation}
\vspace{-6mm}
\end{table}

We see that while $\lambda$ can change the quality of solutions found, the performance is stable within a $\pm \, 3 \times$ regions around $\lambda=0.1$. Further, using dual GD \textit{consistently} outperforms handpicked $\lambda$ with an ELBO threshold $C$ slightly (e.g., $5\%$) higher than the $\mathcal{L}_\text{ELBO}$ achieved without compression, which can be obtained with minimal prior knowledge about the task.
Hence $\lambda$ likely can be tuned without strong prior knowledge. This contrasts with the other methods such as \citet{kipf2019compile} which require knowing the number of skills in each demonstration to perform well while achieving the similar or better performance.

\section{Architecture}
\label{app:arch}
\subsection{Frame Prediction}
Our base architecture is similar to \citet{kim2019variational}\footnote{\url{https://github.com/taesupkim/vta}} and we highlight the difference {\color{blue} blue}. The parameters of the posteriors are implemented with neural networks.
First, the observation $\rvx_{1:T}$ are encoded into a lower dimension embedding $\hat{\rvx}_{1:T}$ with a \texttt{observation encoder}. The embedding is then passed through a \texttt{boundary posterior decoder} that outputs the parameters of $\rvm_{1:T}$. The embedding is further passed through a GRU~\citep{chung2014empirical}, $f_{\text{s-rnn-fwd}}$ which runs on $\hat{\rvx}_{1:T}$ and give the cell state $\texttt{h}_{1:T}$. For $\rvz_{1:T}$, we have two GRU's with cell, $f_{\text{z-rnn-bwd}}$ and $f_{\text{z-rnn-bwd}}$, which run on $\hat{\rvx}_{1:T}$ in different temporal direction (i.e., $f_{\text{z-rnn-bwd}}$ sees the future) to generate cell states $\texttt{c}^{\text{fwd}}_{1:T}$ and $\texttt{c}^{\text{bwd}}_{1:T}$. 
For each time step $t$, a candidate $\rvz_t'$ is sampled from a {\color{blue} $n_z$-dimensional categorical distribution}\footnote{ \citet{kim2019variational} uses a isotropic Gaussian.} whose log probability is a linear projection of $\texttt{c}^{\text{fwd}}_{t} \| \texttt{c}^{\text{bwd}}_{t} $. The sampled $\rvz_t'$ are embedded linearly into a vector of size 128. If $\rvm_1 = 1$, the $\rvz_t = \rvz'_t$; otherwise, $\rvz_t = \rvz_{t-1}$ \footnote{This is the COPY action in \citet{kim2019variational}.}. Then we concatenate $\texttt{h}_{\textcolor{blue}{t-1}} \| \rvz_t$ and linearly project the vector into the mean and standard deviation of a isotropic Gaussian in $\sR^8$. A $\hat{\rvs}_t$ is sampled from this Gaussian.

Th model also keeps two ``belief'' vectors that keep track of the history of the computation. These vectors are modulated by two RNN, $f_{\text{h-rnn}}$ and $f_{\text{c-rnn}}$. The abstraction belief $\rvc_t = \rvm_{t} \cdot f_{\text{c-rnn}}(\rvz_{t}, \rvc_{t-1}) + (1-\rvm_{t}) \cdot \rvc_{t-1}$. The observation belief $\rvh_{t} = \rvm_{t}  \rvc_{t-1} + (1-\rvm_{t}) \cdot f_{\text{h-rnn}}(\hat{\rvs}_t \| \rvc_{t} \| \rvz_t, \rvh_{t-1})$. Finally the state abstraction is computed as the linear projection of $\rvs_t = \text{proj}(\rvh_t \| \hat{\rvs}_t)$. The projection layer is $256 \times 128$. Finally, $\rvs_t$ is passed through a \texttt{observation decoder} and decode into the reconstruction of $\rvx_t$. Refer to \citet{kim2019variational} for more details and an illustration of the computation graph.

For each convolutional layer or transposed convolution layer, the tuple's values correspond to input channel, output channel, kernel size, stride, padding. With that notation in mind, the specific hyperparameters for each components are:
\begin{itemize}
    \item \texttt{observation encoder}: 4-layer convolutional neural network with \{(3, 128, 4, 2, 1), (128, 128, 4, 2, 1), (128, 128, 4, 2, 1), (128, 128, 4, 1, 0)\} with ELU activation~\citep{clevert2015fast} and batch normalization~\citep{ioffe2015batch}. The output is flattened linearly projected to vector of size 128.
    \item \texttt{boundary posterior decoder} is a {\color{blue} 5-layer 1-D causal temporal convolution~\citep{oord2016wavenet} with \{(128, 128, 5, 1, 2) $\times$ 5, (128, 2, 5, 1, 2)\} with ELU activation and batch normalization}. The output is the logits for the $\rvm_{1:T}$.
    \item $f_{\text{z-rnn-bwd}}, f_{\text{z-rnn-fwd}}, f_{\text{s-rnn-fwd}}, f_{\text{c-rnn}}, f_{\text{h-rnn}}$: GRU with cell size 128
    \item \texttt{observation decoder}: 4-layer transposed convolutional layers \{(128, 128, 4, 1, 0), (128, 128, 4, 2, 1), (128, 128, 4, 2, 1), (128, 3, 4, 2, 1)\} with ELU activation and batch normalization on all but the last layer. The last layer has no normalization and has \texttt{tanh} activation.
\end{itemize}

\paragraph{{\color{blue}Removing Boundary Regularization.}}
An important difference from \citet{kim2019variational} is that we do not enforce maximum number of subsequence $N_{\text{max}}$ and the maximum length of subsequence $l_{\text{max}}$ through the boundary prior (This effectively amounts to setting both to a value larger than the sequence length). These two hyperparameters' effect are similar to that of picking the number of segments in CompILE~\citep{kipf2019compile} and provide strong training signal. This assumption is in general problematic since we cannot expect to know a priori what the good values are for these hyperparameters. \citet{kipf2019compile} demonstrates that the performance can suffer if this kind of supervision is wrong. On the other hand, we do not assume anything about the the duration of the subsequence or how many subsequences there are in each demonstration.

We do, however, enforce a minimum length on the skill. While it is largely an optimization choice, we believe this prior is significantly weaker since useful options are almost always temporally extended. Indeed, with only the minimum length constraint, VTA is unable to learn good skills conducive for downstream tasks.

\subsection{Multi-task Grid World Environment}
\label{app:interaction}
For trajectories, we make some larger changes to the architecture to encode properties we want in a model for learning options, e.g., Markov property. Though similar in spirit, a large portion of the architecture is different from the base VTA model, so the difference will not be highlighted.

For the posterior, we have an \texttt{observation encoder} and an \texttt{action encoder} that embeds both the state observation $\rvx_{1:T}$ and actions $\rva_{1:T}$ to get embedding $\texttt{a}_{1:T}$ and $\texttt{x}_{1:T}$ each of size 128. The concatenation $\texttt{a}_{t} \| \texttt{x}_{t}$ linear projected down to a vector of size 128.
The embedding is then passed through a \texttt{boundary posterior decoder} that outputs the parameters of $\rvm_{1:T}$. 
The embedding is further passed through a GRU~\citep{chung2014empirical}, $f_{\text{s-rnn-fwd}}$ which runs on $\hat{\rvx}_{1:T}$ and give the cell state $\texttt{h}_{1:T}$. 
For $\rvz_{1:T}$, we have two GRU's with cell, $f_{\text{z-rnn-bwd}}$ and $f_{\text{z-rnn-bwd}}$, which run on $\hat{\rvx}_{1:T}$ in different temporal direction (i.e., $f_{\text{z-rnn-bwd}}$ sees the future) to generate cell states $\texttt{c}^{\text{fwd}}_{1:T}$ and $\texttt{c}^{\text{bwd}}_{1:T}$.

For sampling $\rvz$ is done differently for interaction data for better gradient and representation learning. Instead of doing straight-through estimator for the categorical random variable, we opt to use vector-quantization~\citep{oord2017neural} with straight-through estimators. First, $\texttt{c}^{\text{fwd}}_{t} \| \texttt{c}^{\text{bwd}}_{t}$ is linearly projected into dimension 128. Then a ReLU is applied and another 128 by 128 linear layer is applied. The VQ codebook is of size $n_{\rvz} \times 128$. Instead of taking argmax of the codebook, we approximate the distribution to be proportional to the softmax over the distance over a temperature $t_{VQ}$. We can sample from this distribution and apply the straight-through gradient on the sampled code candidate $\rvz'_{1:T}$.  If $\rvm_1 = 1$, the $\rvz_t = \rvz'_t$; otherwise, $\rvz_t = \rvz_{t-1}$.

We construct directly $\rvs_t = \rvz_t \| \texttt{x}_t$. and linearly project the vector into the mean and standard deviation of a isotropic Gaussian $\sR^8$. $\rvs_t$ is sampled from this Gaussian and then decoded with the \texttt{action decoder} into the reconstructed action in $\gA$.

For each convolutional layer or transposed convolution layer, the tuple's values correspond to input channel, output channel, kernel size, stride, padding. With that notation in mind, the specific hyperparameters for each components are:
\begin{itemize}
    \item \texttt{observation encoder}: This takes as input the $10 \times 10 \times N_\text{obj} + 2$ and is implemented as a two 2D convolutional layers with ReLU activations, followed by a linear layer with a ReLU activation and a linear layer with no activation.
    The convolutional layers have hyperparameters $(12, 64, 3, 1, 0)$ and $(64, 64, 3, 1, 0)$ respectively.
    The first linear has input dimension $6 \times 6 \times 64$ and output dimension $128$.
    The second linear layer has input and output dimensions $128$.
    \item \texttt{action encoder}: The actions are embedded with an embedding matrix with embedding dimension 128.
    \item \texttt{boundary posterior decoder} is a 5-layer 1-D causal temporal convolution~\citep{oord2016wavenet} with \{(128, 128, 5, 1, 2) $\times$ 5, (128, 2, 5, 1, 2)\} with ELU activation and batch normalization. The output is the logits for the $\rvm_{1:T}$.
    \item $f_{\text{z-rnn-bwd}}, f_{\text{z-rnn-fwd}}, f_{\text{s-rnn-fwd}}$: GRU with cell size 128
    \item \texttt{action decoder}: The decoder is implemented as three linear layers with ReLU activations, followed by a linear layer with no activation.
    The (input, output) dimensions of these layers are as follows: $(128, 128), (128, 128), (128, 128), (128, 5)$, where the last layer outputs logits over the actions.
\end{itemize}

\subsection{3D Navigation Environment}
\label{app:3d_arch}
The architecture used in the 3D navigation environment is identical to that of the multi-task environment detailed above, with the exception that the \texttt{observation encoder} is changed to handle visual inputs.
Specifically, the \texttt{observation encoder} is a a 3-layer convolutional neural network parameters $(3, 32, 5, 2, 0), (32, 32, 5, 2, 0), (32, 32, 4, 2, 0)$, followed two linear layers $(7520, 128), (128, 128)$ with a ReLU activation in between.

\subsection{Hierarchical Reinforcement Learning}
\label{app:hrl_arc}

For all approaches, we parametrize the policy as a double dueling deep Q-network~\citep{mnih2015human, wang2016dueling, van2016deep}.
The parametrization of the Q-function consists of a state embedding followed by two linear layer heads for the state-value $V_\theta(s)$ function and the advantage $A_\theta(s, a)$ function.
Then the Q-function is computed as $Q_\theta(s, a) = V_\theta(s) + A_\theta(s, a) - \frac{1}{|\mathcal{A}|}\sum_{a' \in \mathcal{A}}A_\theta(s, a)$.

The value function $V_\theta(s)$ and advantage function $A_\theta(s, a)$ are computed as linear layers with output dimension $1$ and $|\mathcal{A}|$ respectively on top of the embedding $e(s)$ of the state $s$.
Recall that the state $s$ consists of two portions: the observation $o$ of the grid or 3D environment, and the instruction $i$ corresponding to the next object to pick up, where the instruction is not present in the demonstrations.
The embedding $e(s)$ of the state $s$ is computed by embedding the observation $e(o)$ with the \texttt{observation encoder} defined in the previous sections, and embedding the instruction $e(i)$ with a 16-dimensional embedding matrix.
Then, the embedding $e(s)$ is a final linear layer with output dimension 128 applied to a ReLU of the concatenation of $e(o)$ and $e(i)$.

\section{Hardware}
All our experiments are conducted on a GeForce RTX 2080 Ti or a GeForce RTX A6000.

\section{Hyperparameters}

\subsection{Frame Prediction}
For these set of experiments, we use the same architecture and hyperparameters as \citet{kim2019variational}. The hyperparameters used for both \textit{Simple Colors} and \textit{Conditional Colors} are:
\begin{itemize}
    \item KL divergence weight $\beta: 1.0$
    \item MDL objective weight $\lambda \text{ (fixed)}: 0.1$
    \item Mini-batch size: $512$
    \item Training iteration: 30000
    \item Number of skills $n_z = 10$
\end{itemize}

\subsection{Multi-task Grid World Segmentation}
We use the following hyperparameters for skill learning from demonstrations on the multi-task grid world environment experiments.
\begin{itemize}
    \item KL divergence weight $\beta = 0$:
    \begin{itemize}
        \item We found that the compression objective readily provides strong and sufficient regularization for the latent code. Adding additional KL divergence often results in the model not being able to reconstruct the actions properly.
    \end{itemize}
    \item MDL objective weight $\lambda$:
    \begin{itemize}
        \item We use a adaptive scheduling that approximate dual gradient descent. $\lambda$ is initialized to $0$. After every gradient step, $\lambda$ is increased by $2.0 \times 10^{-5}$ if $\gL_{\text{ELBO}} \leq 0.05$ (Since $\beta = 0$, this is effectively $\gL_{\text{rec}}$); otherwise, $\lambda$ is decreased by $2.0 \times 10^{-5}$. After each update, the value of $\lambda$ is clipped to $[0, 0.05]$. We find this setting provides the most stable training.
    \end{itemize}
    \item Mini-batch size: $64$
    \item Training iteration: 20000
    \item Number of skills $n_z = 10$
    \item $t_{\text{VQ}} = 0.1$
    \item Learning rate: $0.0005$ with Adam Optimizer
\end{itemize}

\subsection{3D Visual Navigation Segmentation}
We use the same hyperparameter as Multi-task grid world experiments except that the maximum $\gL_\text{ELBO}$ is capped at $0.001$, \textit{i.e.}, $\gL_\text{ELBO} \leq 0.001$.

\subsection{Hierarchical Reinforcement Learning}
\label{app:hrl-hparam}

We use the same hyperparameters for training the policy for all approaches.
Specifically, these hyperparameter values are as follows:
\begin{itemize}[itemsep=0pt]
    \item $\epsilon$-greedy schedule: We linearly decay $\epsilon$ from $1$ to $0.01$ over $500K$ timesteps for dense reward settings and over $5M$ timesteps for sparse reward settings in the multi-task grid world environment. In the 3D visual navigation environment, we decay over $250K$ timesteps.
    \item Discount factor $\gamma$: We use $\gamma = 0.99$ in all of the DQN updates.
    \item Maximum replay buffer size: 50K
    \item Minimum replay buffer size before updating: 500
    \item Learning rate: 0.0001 with the Adam optimizer~\citep{kingma2014adam}
    \item Batch size: 32
    \item Update frequency: every 4 timesteps
    \item Target syncing frequency: every 50K updates for multi-task grid world environment, every 30K updates for 3D navigation
    \item Gradient $\ell_2$-norm clipping: 10
    \item Marginal threshold $\alpha = 0.001$
\end{itemize}

Since the demonstrations do not contain the instruction list observations, we set those to 0 in the demonstrations for the behavior cloning approach.
Though we use $\alpha = 0.001$ in all of the experiments for consistency, we also found that slightly higher values of $\alpha$ yielded greater sample efficiency for \ours.

\section{Algorithm}\label{app:algo}
The Lagrangian for the unconstrained optimization problem is:
\begin{align*}
    \min_{\boldsymbol{\theta}}\max_{\lambda \geq 0}\,\, \lambda \gL_{\text{CL}}(\boldsymbol{\theta}) +  \left(\gL_{\text{ELBO}}(\boldsymbol{\theta}) - C \right).
\end{align*}
In standard KKT condition, the dual variable is introduced on the ELBO, but the two problems are equivalent by inverting the dual variable. Since the only constraint on $\lambda$ is that it is non-negative, this transformation does not change the optimal solution. Further note that we find that in many cases, choosing a fixed constant value for $\lambda$ is sufficient for solving the problem.

\begin{algorithm}[ht]
\caption{\textbf{Learning Options via Compression} (LOVE):
We highlight differences between our method and prior work (VTA) in {\color{blue} blue}.
}\label{alg:love}
\begin{algorithmic}[1]
\State Initialize model parameters $\boldsymbol{\theta}$
\While{not converged}
    \State Sample a trajectory $(\rvx_0, \rva_0, \ldots, \rvx_T)$ from the pre-collected experience $\gD$
    \For{$t = 1, 2, \ldots, T$}
    \State Sample boundary conditioned on entire trajectory $m_{t} \sim q_{\boldtheta}(m_t \mid \rvx_{1:t})$
    \State Sample skill from skill posterior $z_t \sim q_{\boldtheta}(z_t \mid m_t, \rvx_{1:t}, \rva_{1:t})$
    \State Sample abstraction from abstraction posterior $\rvs_{t} \sim q_{\boldtheta}(\rvs_t \mid z_t, \rvx_{t})$
\State Compute probability of the correct action from the decoder $p_{\boldtheta}(\rva_t \mid \rvs_t)$
\EndFor
\State Compute lower bound on maximum likelihood objective $\gL_{\text{ELBO}}$ according to Equation \ref{eqn:rl_elbo}
\State {\color{blue} Compute compression objective $\gL_{\text{CL}}$ according to Equation \ref{eqn:compression_objective}}
\State Update $\boldsymbol{\theta} \gets \boldsymbol{\theta} - \eta \nabla(\gL_{\text{ELBO}} + \textcolor{blue}{\lambda \gL_{\text{CL}})}$ and $\textcolor{blue}{\lambda}$
\EndWhile\label{euclidendwhile}
\end{algorithmic}
\end{algorithm}

\begin{algorithm}[H]
\caption{Executing skill $\rvz$}\label{alg:hrl}
\begin{algorithmic}[1]
\While{skill has not terminated}
\State Compute state abstraction $\mu_{\rvs_t} = \E_{s \sim q_{\mathbf{\theta}}(\rvs_t \mid \rvz, \rvx_t)}\left[s\right]$
\State Take action $\rva_t = \argmax_{\rva} \,\, p_{\mathbf{\theta}}(\rva \mid \mu_{\rvs_t})$
\State Observe next state $\rvx_{t + 1}$
\State Terminate if $\E_{\rvm_{t + 1} \sim q_{\mathbf{\theta}}(\rvm_{t + 1} \mid \rvx_{1:t + 1}) } \left[\rvm_{t + 1}\right] > 0.5$
\EndWhile\label{euclidendwhile}
\end{algorithmic}
\end{algorithm}

\newpage
\section{Visualization of Learned Skills}
\label{app:visualization}

\begin{figure*}[ht]
\begin{center}
\vspace{-1mm}
    \centerline{\includegraphics[width=\textwidth]{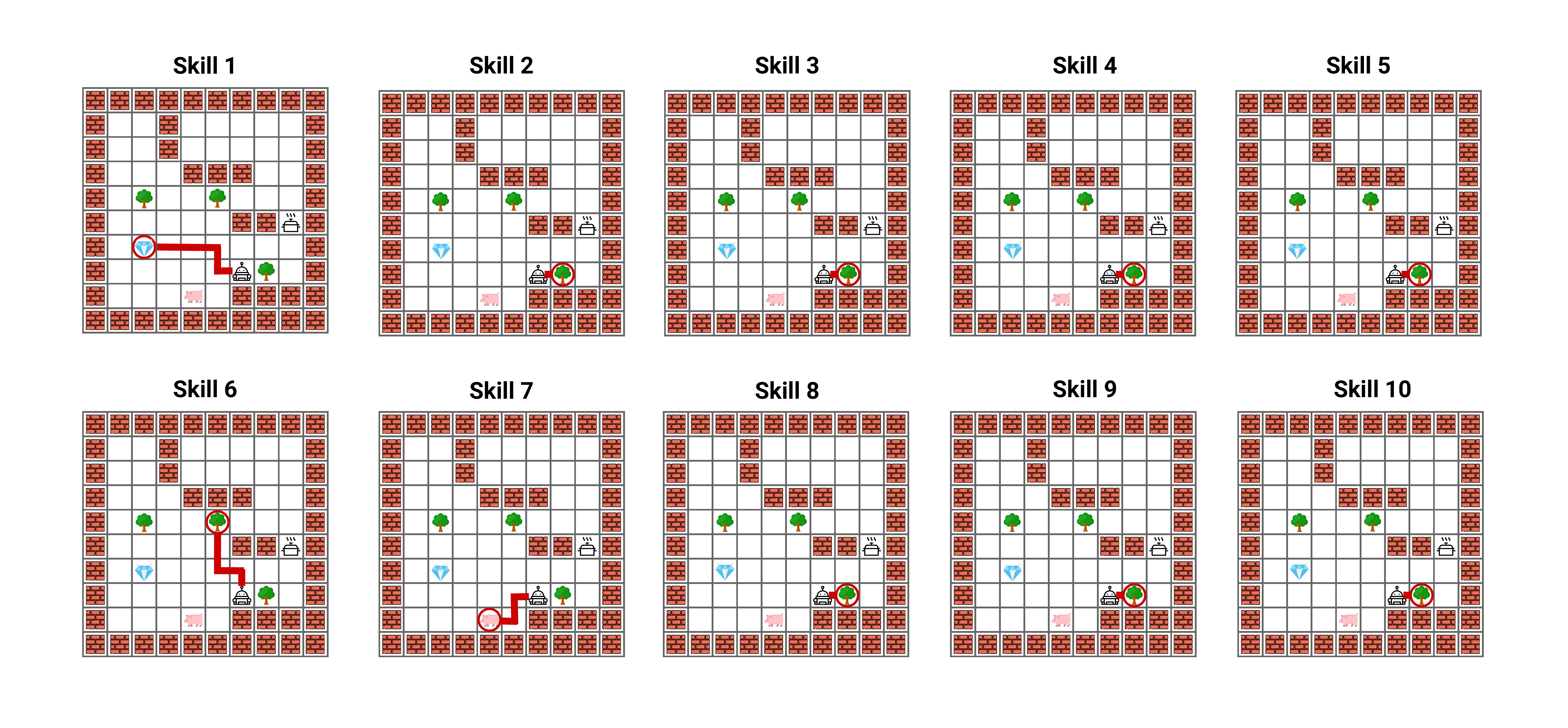}}
\vspace{-4mm}
\caption{Skill visualization for all 10 skills of \ours.}
\label{fig:mdl-vis}
\end{center}
\vspace{-5mm}
\end{figure*}

\begin{figure*}[h!]
    \centering
    \includegraphics[width=\textwidth]{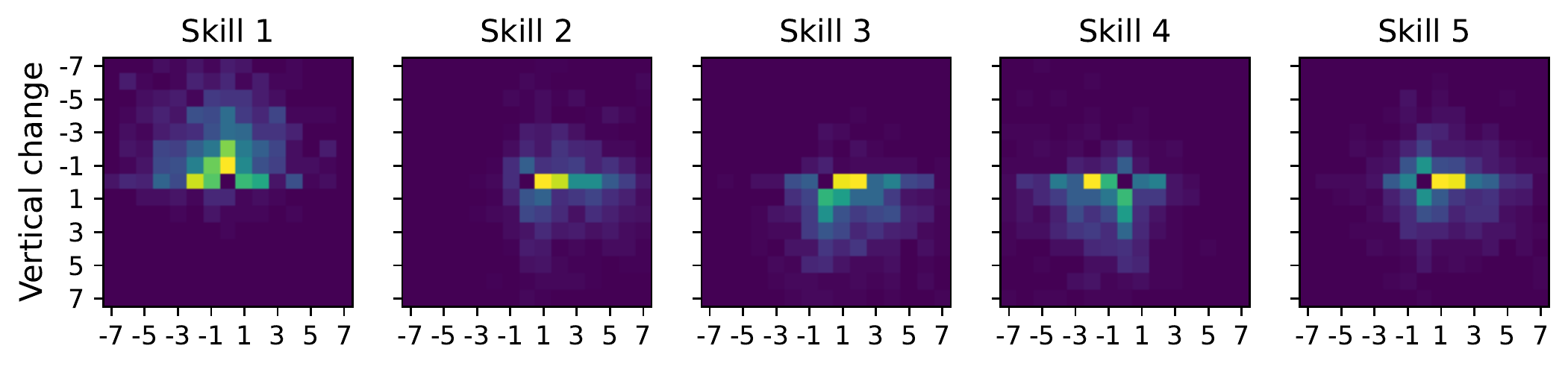}
    \includegraphics[width=\textwidth]{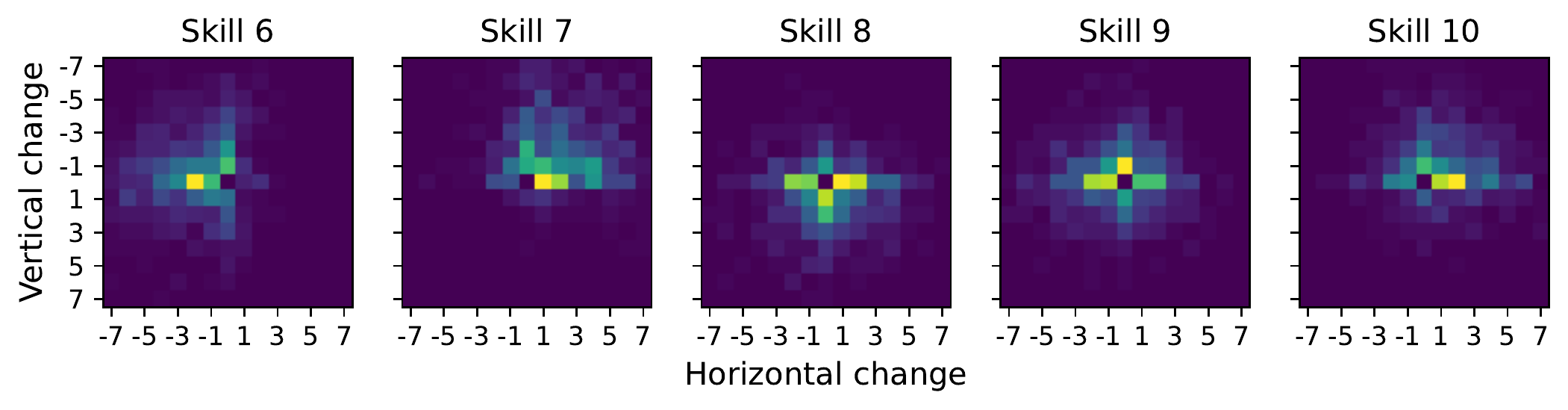}
    \vspace{-5mm}
    \caption{Heat map the agent's end position minus the agent's start position after applying each skill. In other words, this shows what direction the agent moves in after following each skill.}
    \label{fig:heatmap}
\end{figure*}

In Figure \ref{fig:mdl-vis}, we visualize the behavior of following each of \ours's 10 skills on an example task.
Each skill moves to and picks up an object, and the total set of skills covers 3 of the 4 object types in the task.
These skills pick up both nearby objects, such as the tree in skill 2, as well as more distant objects, such as the diamond in skill 1.
This contrasts DDO and VTA, which do not learn skills that move to and pick up objects.

To further understand what each of \ours's skills does, we analyze whether there is a correlation between each skill and either the type of object it picks up, or the location of the object it picks up.
We find that there appears to be little correlation between each skill and the type of object it picks up.
Plotting the frequency of picking up each object type by skill shows that each skill picks up each object type with roughly the same frequency.
Instead, there appears to be a correlation between each skill and the location of the object it picks up, illustrated in Figure \ref{fig:heatmap}.
For example, skill 1 appears to specialize in moving to and picking up objects that are above the agent, while skill 7 tends to pick up objects that are up and to the right of the agent. It is also interesting to note that some skills are much more biased to move in certain direction (e.g., skill 1, 6, 7) while some appear to be more general (e.g., 4, 5, 8, 9) and move in any direction.
Note in Figure \ref{fig:mdl-vis}, it may seem \ours's skills seem short / redundant.
This is due to the fact that \ours's skills each pick up an object, so they naturally appear short when the agent is close to the objects, as in Figure \ref{fig:mdl-vis}.
However, when the agent is far away from the objects, the skills still pick up the objects and are much longer. From Figure~\ref{fig:heatmap}, we can see from the heatmap that even the skills that appear to do same thing in Figure~\ref{fig:mdl-vis} have very distinct behaviors depending on the state they are in.
Also, recall \ours imposes a sparse distribution over skills (Section~\ref{sec:hrl}). After filtering, skill $3$, $7$ and $10$ in Figure~\ref{fig:mdl-vis} are dropped since they have low marginal probability and are not used to describe a significant part of the trajectories.
Figure \ref{fig:mdl-vis} includes redundant skills \ours does not use; the used skills in the sparse distribution have little redundancy.

\section{Option Critic Results}
\begin{figure*}[h]
\begin{center}
\vspace{-1mm}
    \centerline{\includegraphics[width=0.4\linewidth]{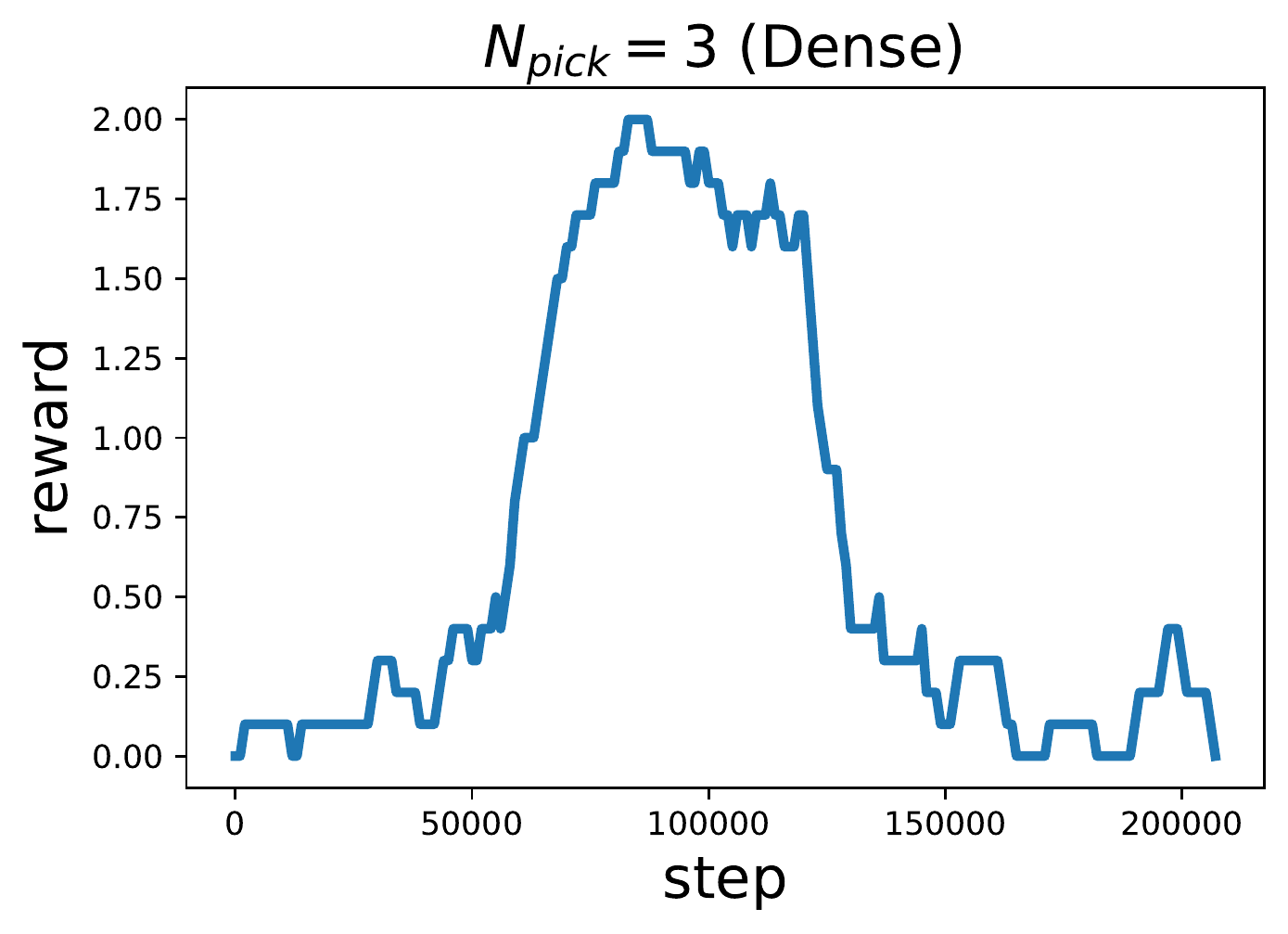}
    \includegraphics[width=0.4\linewidth]{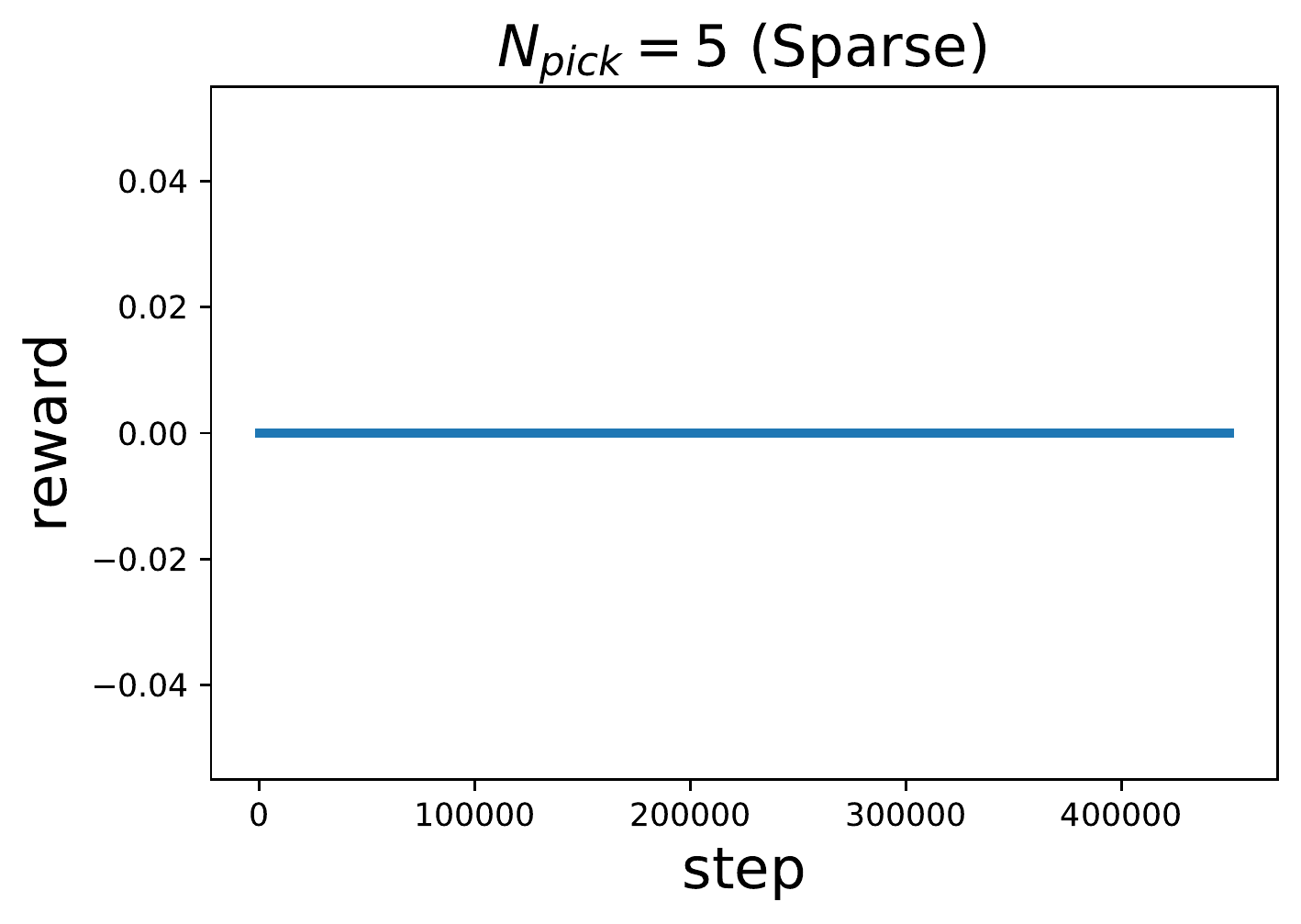}}
\vspace{-2mm}
\caption{Results of running option critic~\citep{bacon2017option} on 2 of the settings we considered in Figure~\ref{fig:grid_returns}.}
\label{fig:oc}
\end{center}
\end{figure*}

Option critic~\citep{bacon2017option} is a classical online HRL algorithm. We show the results of running option critic\footnote{\url{https://github.com/lweitkamp/option-critic-pytorch}} on two of the four RL tasks we considered in Figure~\ref{fig:oc}. The best return achieved over the training is shown in Table~\ref{tab:op-reward}. We set the number of option to $8$ following \citet{bacon2017option} and leave other hyperparameters untouched. $N_{\text{pick}}=3$ with dense reward is the easily setting and $N_{\text{pick}}=5$ with sparse reward is the hardest setting. We see that in $N_{\text{pick}}=3$, option critic is able to reach reward of $2$ (maximum possible is $3$) at around 1M environment steps but the performance deterioates afterwards. In $N_{\text{pick}}=5$, the algorithm fails to make any meaningful progress. These observations may suggest online HRL algorithms like option critic may be insufficient for solving these tasks.

\begin{table}[h]
\centering
\begin{tabular}{ccc}
    \toprule
     & $N_{\text{pick}} = 3$ (Dense) &  $N_{\text{pick}} = 5$  (Sparse)  \\
     \cmidrule(lr){2-2} \cmidrule(lr){3-3}
    Option-Critic & $2.0$ & $0.0$ \\
    \ours (Ours) & $\mathbf{3.0}$ & $\mathbf{0.7}$ \\
    \bottomrule
\end{tabular}
\vspace{3mm}
\caption{Comparison between \ours and \citet{bacon2017option} on the two RL tasks we consider in this work. Each entry is the maximum average return achieved during the course of training for both algorithms. For $N_{\text{pick}} = 3$ (Dense), $3.0$ is the maximum possible return. For $N_{\text{pick}} = 5$ (Sparse), $1.0$ is the maximum possible return.}
\label{tab:op-reward}
\end{table}

\section{Comparison to \citet{zhang2021minimum}}
\begin{table}[h]
\small
\centering
\begin{tabular}{ccccccc}
    \toprule
    &\multicolumn{2}{c}{\textit{Simple Colors}}& \multicolumn{2}{c}{\textit{Conditional Colors}}&\multicolumn{2}{c}{\textit{Navigation}}\\
     & \textbf{MDL} & \textbf{\ours} & \textbf{MDL} & \textbf{\ours}  & \textbf{MDL} & \textbf{\ours}  \\
     \cmidrule(lr){2-2} \cmidrule(lr){3-3} \cmidrule(lr){4-4} \cmidrule(lr){5-5} \cmidrule(lr){6-6} \cmidrule{7-7} 
    Precision & $0.87 $ & $\mathbf{0.99} $ & $0.84 $ & $\mathbf{0.99}$ & $0.79 $ & $\mathbf{0.90} $\\
    Recall & $0.78$ & $\mathbf{0.85}$ & $0.82 $ & $\mathbf{0.83} $& $0.34 $& $\mathbf{0.94} $ \\
    F1 & $0.82 $ & $\mathbf{0.91}$ & $0.83 $ & $\mathbf{0.90}$ & $0.48$ & $\mathbf{0.92}$ \\
    \bottomrule
\end{tabular}
\vspace{3mm}
\caption{Comparison between \ours and \citet{zhang2021minimum} on the three segmentation tasks we consider in this work. We refer to \citet{zhang2021minimum} as \textbf{MDL} in the table.}
\label{tab:zhang}
\end{table}

\citet{zhang2021minimum} learn open-loop skills, which do not condition on the state and therefore cannot adapt to different states. Further, the MDL used in \citet{zhang2021minimum} is equivalent to the variational inference used by VTA (on a different graphical model), which can be seen as greedily compressing each skill independently, rather than compressing a whole trajectory as \ours does. Table~\ref{tab:zhang} reports segmentation results on the Color domain and grid world, akin to Tables~\ref{tab:grid_prec_recall} and \ref{tab:lambda_ablation}. \citet{zhang2021minimum} performs the same as VTA on the Color domains -- as they both use variational inference and there is no state -- which achieves lower precision / recall than \ours. On the grid world navigation, \citet{zhang2021minimum} fails to recover the boundaries, unlike \ours, because skills that navigate to objects require observing the state. This problem is shared by all open-loop approaches.

\section{Number of Initial Skills Ablation}
\begin{table}[h]
\small
\centering
\begin{tabular}{cccccccc}
    \toprule
     & $K=2$ & $K=5$ & $K=10$ & $K=15$  & $K=20$ & $K=30$ & $K=50$  \\
     \cmidrule(lr){2-2} \cmidrule(lr){3-3} \cmidrule(lr){4-4} \cmidrule(lr){5-5} \cmidrule(lr){6-6} \cmidrule(lr){7-7} \cmidrule{8-8} 
    Precision & $0.27 $ & $0.80 $ & $0.90 $ & $0.96$ & $0.96$ & $0.95$ & $0.95$\\
    Recall & $0.53$ & $0.91$ & $0.94$ & $0.96$ & $0.95$& $0.96$ & $0.93$ \\
    F1 & $0.35$ & $0.86$ & $0.92 $ & $0.96$ & $0.95$ & $0.95$ & $0.94$\\
    \bottomrule
\end{tabular}
\vspace{3mm}
\caption{Performance of \ours on the grid navigation task with varying number of initial skills (K).}
\label{tab:skill_ablation}
\end{table}

The number of skills to extract from demonstrations is often not known a priori, so choosing an appropriate value of the number of initial skills before filtering $K$ is important.
Intuitively, we hypothesize that $K$ can be set conservatively: \ours requires at least a minimum number of skills in order to fit the behaviors in the demonstrations, but may be able to gracefully prune out skills if $K$ is set too high via the filtering described in Section~\ref{sec:hrl}.
We test this intuition by varying $K \in \{2, 5, 10, 15, 20, 30, 50\}$ and measuring the segmentation performance on the grid world multi-task domain, and find that it appears to hold true in Table~\ref{tab:skill_ablation}.
For smaller values of $K$, such as $K = 2$ and $K = 5$, \ours's F1 scores significantly degrade.
However, performance remains high for all values where $K$ is sufficiently large.
Hence, we conservatively sizing $K$ to be a large number.

\section{Comparison to Different Regularizers}
\begin{table}[h]
\small
\centering
\begin{tabular}{cccccccc}
    \toprule
     & \texttt{entropy} & \texttt{num switch} & \texttt{VTA(3,5)} & \texttt{VTA(3,10)}  & \texttt{VTA(5,5)} &  \texttt{VTA(5,10)} & \ours \\
     \cmidrule(lr){2-2} \cmidrule(lr){3-3} \cmidrule(lr){4-4} \cmidrule(lr){5-5} \cmidrule(lr){6-6} \cmidrule(lr){7-7} \cmidrule{8-8} 
    Precision & $0.26 $ & $0.80 $ & $0.34 $ & $0.40$ & $0.34$ & \textbf{0.95} & 0.90\\
    Recall & $0.53$ & $0.93$ & $0.46$ & $0.60$ & $0.50$& $0.43$ & \textbf{0.94}\\
    F1 & $0.35$ & $0.86$ & $0.39$ & $0.48$ & $0.40$ & $0.37$ & \textbf{0.92} \\
    \bottomrule
\end{tabular}
\vspace{3mm}
\caption{Comparison of different kinds of regularizers that have been used in the literature for skill segmentation. \ours outperforms all significantly.}
\label{tab:regularizer}
\end{table}

Several prior methods encourage skills to act for more timesteps or prevent skills from switching too frequently, which can similarly help avoid degenerate solutions like \ours.
We note that such prior methods do not necessarily yield skills that help learn new tasks, while the maximum of \ours's objective achieves the information-theoretic limit of compressing the trajectories (under a fixed function class), which is not done in other works.
We empirically compare \ours with several such methods on the segmentation of the multi-task grid world domain, including:

\begin{enumerate}
    \item \emph{Entropy}: This approach regularizes the entropy of skill marginal akin to~\citet{Shankar2020Discovering}, which they refer to this as \textit{parsimony}. Specifically, this approach adds a regularization term $\gH_{p_{\rvz}^{\star}}[\rvz]$ to the objective $\mathcal{L}_\text{ELBO}$.
    \item \emph{Num switches}: This approach penalizes switching between skills, similar to~\citet{harb2018waiting}, a variant of the Option-Critic framework~\citep{bacon2017option}. Specifically, this approach adds a penalty $n_\text{s}$ to the objective $\mathcal{L}_\text{ELBO}$ equal to the number of switches (i.e., the number of times $m_t = 1$) in a demonstration.
    \item VTA($N_\text{max}, l_\text{max})$: VTA includes a prior on its boundary variables, which effectively sets the maximum number of skills per episode to be $N_\text{max}$ and sets the maximum number of actions a skill can take to $l_\text{max}$. We experiment with several settings of $N_\text{max}$ and $l_\text{max}$, including those that use prior knowledge about the domain not used by \ours.
\end{enumerate}

The results are summarized in Table~\ref{tab:regularizer}.
We find that regularizing the number of switches performs fairly well, but \ours achieves higher F1 than other approaches that regularize the skills.

\end{document}